%% file: neurips_2026.tex
\documentclass{article}

\usepackage[preprint]{neurips_2026}

\usepackage[utf8]{inputenc}
\usepackage[T1]{fontenc}
\usepackage{hyperref}
\usepackage{url}
\usepackage{booktabs}
\usepackage{amsfonts}
\usepackage{amsmath}
\usepackage{nicefrac}
\usepackage{microtype}
\usepackage[dvipsnames]{xcolor}
\usepackage{graphicx}
\graphicspath{{assets/figures/}}
\usepackage{enumitem}
\usepackage{pifont}
\usepackage{bm}
\usepackage[dvipsnames,table]{xcolor}
\usepackage[most]{tcolorbox}
\usepackage{titletoc}
\usepackage{amsthm}

\definecolor{boxbg}{HTML}{E1ECEB}
\definecolor{boxframe}{HTML}{205EA6}

\definecolor{flexblue}{HTML}{205EA6}
\definecolor{flexred}{HTML}{AF3029}
\definecolor{flexorange}{HTML}{BC5215}
\definecolor{flexgreen}{HTML}{66800B}

\providecommand{\eqblue}[1]{\textcolor{flexblue}{#1}}
\providecommand{\eqred}[1]{\textcolor{flexred}{#1}}
\providecommand{\eqorange}[1]{\textcolor{flexorange}{#1}}

\definecolor{blue}{rgb}{0,0.2,0.5}
\definecolor{green}{rgb}{0.1,0.35,0.0}
\definecolor{red}{rgb}{0.5,0.0,0.0}
\definecolor{purple}{rgb}{0.4,0,0.6}
\definecolor{cyan}{rgb}{0.0,0.4,0.3}
\definecolor{orange}{rgb}{0.6,0.4,0.0}
\definecolor{gray}{rgb}{0.3,0.3,0.3}

\hypersetup{%
  colorlinks,
  linkcolor=cyan,
  citecolor=blue,
  urlcolor=blue
}

\theoremstyle{plain}
\newtheorem{proposition}{Proposition}[section] 
\newtheorem{observation}{Observation}[section]

\theoremstyle{definition}

\theoremstyle{remark}

\title{Flowing with Confidence}

\author{%
     Friso de Kruiff$^{1,2}$\thanks{Correspondence to \texttt{f.c.dekruiff@uva.nl}.} \quad
     Dario Coscia$^{3,2}$\thanks{Work done in collaboration with the University of Amsterdam.} \quad
     Max Welling$^{1,2}$ \quad
     Erik Bekkers$^{2}$ \\[0.4em]
     $^{1}$CuspAI \quad $^{2}$AMLab, University of Amsterdam \quad $^{3}$mathLab, SISSA
   }

\begin{document}

\maketitle

\begin{abstract}
\input{sections/abstract}
\end{abstract}
\input{sections/introduction}
\input{sections/related_work}
\input{sections/method}
\input{sections/experiments}
\input{sections/conclusion}

\newpage
\input{sections/acknowledgements}
\bibliographystyle{plainnat}
\bibliography{references}


\newpage
\appendix
\input{sections/appendix}

\end{document}

%% file: sections/abstract.tex
Generative models can produce nonsensical text, unrealistic images, and unstable materials faster than simulation or human review can absorb; without per-sample confidence, trust erodes. Existing fixes run $k$ ensembles or stochastic trajectories at $k\times$ compute, measuring variability between models, not model confidence. We propose Flow Matching with Confidence (FMwC). FMwC injects input-dependent multiplicative noise at selected layers, propagates its variance through the network in closed form, and integrates it along the ODE trajectory, yielding a per-sample confidence score at standard sampling cost. The score supports multiple uses: filtering improves image quality and thermodynamic stability of crystals; editing rewinds trajectories to the points where the model commits and redirects them; and adaptive stepping concentrates ODE compute where the flow is ambiguous. We find that the confidence score correlates with the magnitude of the divergence of the learned velocity field, which gives us a window to understand the generative process, opening up surgical forms of guidance that target the moments that matter, new sampling algorithms and interpretability of generative models.

%% file: sections/introduction.tex
\section{Introduction}
\label{sec:introduction}

Generative models have become a central pillar of modern machine learning, producing samples across language, vision, and the natural sciences. Within this class, flow matching has emerged as a leading approach, training a deterministic velocity field by simulation-free regression and integrating it at inference to produce samples \citep{lipman2022flow}. Despite their broad success~\citep{miller2024flowmm, gat2024discrete, qin2024defog, eijkelboom2024variational}, state-of-the-art models still generate outputs that contain artefacts, and each sample arrives with no native indication of how reliable it is: a hallucinated citation, an anatomically broken figure, and a thermodynamically impossible crystal each emerge with the same implicit confidence. The cost of this missing signal compounds with deployment: in materials discovery, for instance, validating a single generated candidate can take days of compute or weeks of laboratory time, and indiscriminate generation wastes both. This motivates the central question of this work:
\begin{quote}
    \centering
    \emph{Can flow-matching models learn to know what they don't know?}
\end{quote}

\begin{figure}[t]
\centering
\includegraphics[width=\textwidth]{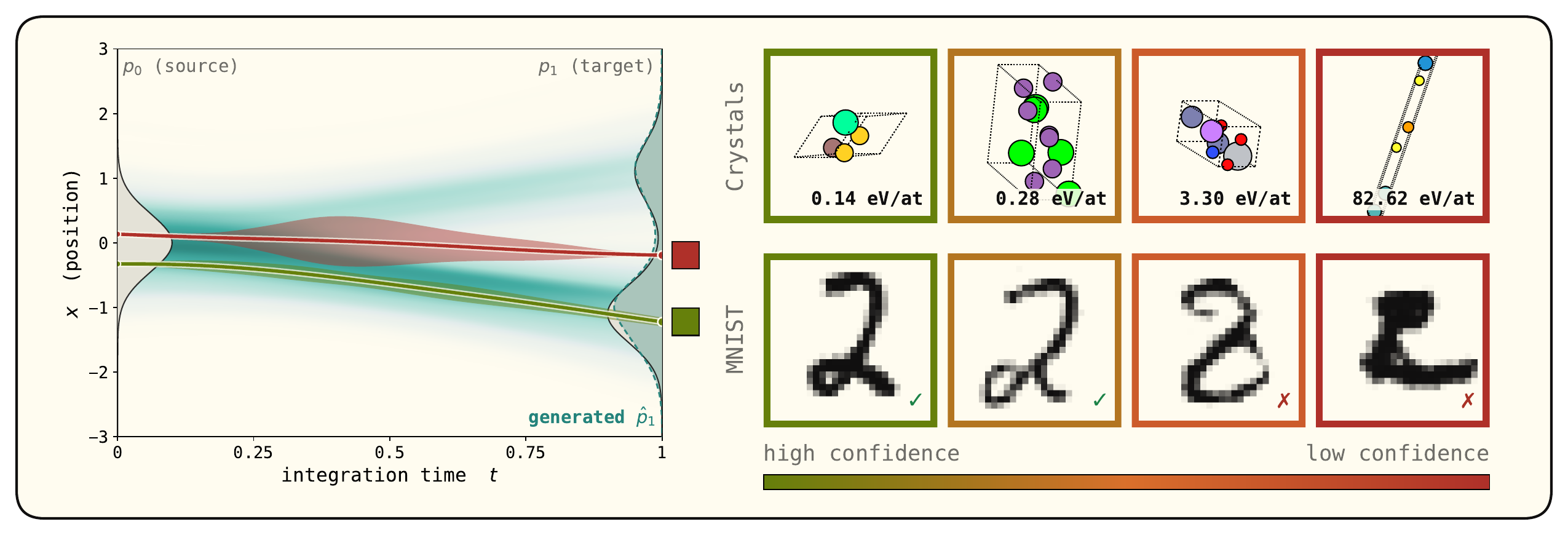}
\caption{\textbf{Flow-Matching model with Confidence.}
  \textbf{(Left)}~Variance is propagated alongside the sample as it is generated along the path $p_0 \!\to\! p_1$: well-placed samples (\textbf{\textcolor[HTML]{66800B}{green}}) contract into a tight $\pm\sigma$ band, while misplaced ones (\textbf{\textcolor[HTML]{AF3029}{red}}) stay diffuse.
  \textbf{(Right)}~Sorting by the resulting confidence score recovers quality across modalities---low-energy crystals and clean digits at high confidence, implausible structures and malformed digits at low confidence.}
\label{fig:method_overview}
\end{figure}
This question echoes \citet{nalisnick2019deep}, who showed that the log-likelihoods of deep generative models cannot reliably distinguish in- from out-of-distribution inputs. We ask an orthogonal question on the generation side: not whether the model can flag anomalous inputs from its likelihood, but whether it can produce a per-sample reliability signal for the outputs it itself generates, a signal that does not require access to a likelihood at all.

Confidence estimation~\citep{berry2023efficient, jazbec2025generative} offers a route to such a signal, yet existing approaches force a costly trade-off between training and inference budgets. Deep ensembles \citep{lakshminarayanan2017simple} train $k$ models; MC-Dropout \citep{gal2016dropout}, SDE-style perturbation, and post-hoc last-layer Laplace methods \citep{jazbec2025generative} instead run $k$ trajectories per output at inference. Beyond their cost, all of these report \emph{disagreement between models or between runs} rather than a property of the single model that produced the output, leaving per-sample confidence either expensive, or unavailable (Table~\ref{tab:overview}).

We propose \textbf{Flow Matching with Confidence} (FMwC), a flow-matching model that produces a calibrated per-sample confidence score for every generated output, at the same training and inference cost as the deterministic baseline. The construction builds on variational dropout \citep{kingma2015variational} with input-dependent rates \citep{coscia2025barnn}, and departs from prior variational-dropout work by propagating variance through the entire ODE trajectory rather than through the network at a single time step; that is what turns per-layer learned noise into a trajectory-level confidence score. The resulting score is a within-model, geometric signal that reflects how the learned velocity field responds to perturbations of its own hidden states, distinct from the between-model disagreement that ensembles and MC-Dropout measure. The main contributions of this paper are:
\begin{enumerate}[leftmargin=*,topsep=2pt,itemsep=2pt]
\item \textbf{A single-pass confidence estimator for flow matching.} We introduce Flow Matching with Confidence (FMwC), a flow-matching model that emits a per-sample, trajectory-level confidence score alongside every generated sample, from a single deterministic forward integration as the standard flow-matching baseline. The score is obtained by analytically propagating learned, input-dependent layer-wise noise through the ODE, yielding a closed-form trajectory variance with no Monte Carlo sampling and no auxiliary networks.

\item \textbf{Confidence enables practical, training-free improvements at inference.} Filtering by FMwC monotonically improves sample quality over the deterministic baseline. Second, the temporal variance peak identifies the exact moment the model commits to a sample enabling targeted, constraint-respecting edits without retraining, while confidence score adaptive integration recovers quality at low step budgets where uniform stepping degrades sharply. The same signal supports zero-shot quality selection in confidence-conditioned generation, and serves as a single-pass proxy for flow diagnostics that would otherwise require multiple integrations. 

\item \textbf{Confidence reflects the intrinsic geometry of the learned flow.} We show that FMwC's single-pass score tracks the local divergence of the learned velocity field. This establishes the confidence score as a within-model geometric quantity, a property of the single trained flow rather than a proxy for inter-model disagreement.
\end{enumerate}

%% file: sections/related_work.tex
\section{Related Work}
\label{sec:related_work}

\begin{table}[t]
\centering
\caption{Comparison of uncertainty quantification approaches for flow matching. FMwC is the only method that provides per-sample confidence at the cost of a single deterministic forward pass, matching the inference budget of standard flow matching.}
\label{tab:overview}
\begin{tabular}{lccc}
\toprule
Method & Trained models & Sampled trajectories & Confidence \\
\midrule
FM                   & $1$ & $1$ & \textcolor{red}{\ding{55}} \\
MC-Dropout           & $1$ & $k$ & \textcolor{ForestGreen}{\ding{51}} \\
Ensemble             & $k$ & $k$ & \textcolor{ForestGreen}{\ding{51}} \\
\textbf{FMwC (ours)} & $1$ & $1$ & \textcolor{ForestGreen}{\ding{51}} \\
\bottomrule
\end{tabular}
\end{table}

\paragraph{Flow matching has a unique target with intrinsically delicate geometry.}
Flow matching trains a velocity field by simulation-free regression against a conditional target \citep{lipman2022flow}, and has been extended to general geometries \citep{chen2024flow} and graph-structured data \citep{eijkelboom2024variational}. Because the conditional and marginal flow-matching gradients coincide, the optimal solution is uniquely determined \citep{lipman2022flow}. It is a conditional expectation: at each intermediate state, the velocity averages over the conditional velocities of all endpoints reaching that state \citep{pooladian2023multisample}. This averaging structure renders the optimal field geometrically delicate: it undergoes spontaneous symmetry-breaking transitions at critical times, where the trajectory's posterior over endpoints contracts onto a single mode and the susceptibility of the velocity field diverges \citep{raya2023spontaneous, ambrogioni2024thermodynamics}, with the local divergence of the field further governing how information concentrates along the trajectory \citep{stancevic2026information}. A finite-capacity network trained on finite data approximates such a field imperfectly, with the approximation expected to be least reliable in regions of intrinsic instability. The standard flow-matching architecture, however, produces a deterministic velocity at inference and exposes no per-sample signal of where the learned flow is trustworthy.

\paragraph{Per-sample reliability methods at inference are expensive or measure the wrong quantity.}
Likelihood is not a per-sample reliability signal for deep generative models \citep{nalisnick2019deep, kamkari2024geometric}, motivating methods that produce one explicitly. Deep ensembles \citep{lakshminarayanan2017simple} train $k$ independent networks; at inference, each generated sample requires $k$ separate trajectory integrations, scaling the per-sample cost by $k$ at every Euler step. MC-Dropout \citep{gal2016dropout} reuses a single trained model but still requires $k$ stochastic trajectories per output. The post-hoc last-layer Laplace family \citep{jazbec2025generative} draws $k$ weight realisations and runs $k$ full sampling trajectories per output. Each of these methods reports disagreement between models, between runs, or between externally embedded outputs, and demanding significantly more compute than standard flow matching (Table~\ref{tab:overview}). A reliability signal at the inference cost of standard flow matching has not been demonstrated yet.

\paragraph{Toward efficient uncertainty estimation.}
A separate line of work focuses on efficient uncertainty estimation. Variational Adaptive Dropout (VAD) learns a small auxiliary network to predict the stochastic perturbation scale for each network' weights, so the posterior broadens where the input is ambiguous and tightens where it is not \citep{coscia2025barnn, coscia2025blips}. Crucially, under VAD the predictive distribution at each linear layer is Gaussian and available in closed form, but stacking layers with nonlinearities breaks this closed form, needing multiple trajectory samples to obtain moments estimates. Sampling-free moment propagation closes exactly this gap: it pushes moments through pointwise nonlinearities analytically, by local linearisation \citep{wang2013fast} or by quadrature \citep{postels2019sampling}, recovering a closed-form predictive distribution across the full depth of the network in a single pass. Combining input-dependent variational posteriors with closed-form moment propagation through every layer has not, to our knowledge, been carried into a trajectory-valued generative ODE.

%% file: sections/method.tex
\section{Method}
\label{sec:method}

FMwC is a single-pass approach for equipping any flow-matching model with a per-sample confidence signal. It treats the network's weights stochastically by injecting input-dependent Gaussian noise, and propagates the resulting variance analytically through the ODE alongside the deterministic trajectory. In the rest of this section, we formalise FMwC as a variational posterior over velocity fields, derive the variance recursion through the network and the integrator, and define scalar readouts that summarise the trajectory variance and can be used as confidence scores.

\subsection{Flow Matching with Confidence}
\label{sec:method-fmwc}

\paragraph{Flow matching.}
Flow matching~\citep{lipman2022flow} learns a continuous-time transformation from a base distribution $p_0$ to a target density $p_1 := p_{\mathrm{data}}$ by training a time-dependent velocity field $v_{t,\bm{\theta}}(\bm{x}_t)$ that transports $p_0$ to $p_1$ along a prescribed probability path, where $\bm{\theta}$ are the trainable network parameters. At inference, samples are drawn by integrating the learned field from $t{=}0$ to $t{=}1$. Standard flow matching is silent about the reliability of each step: the network outputs a single value $\mu_t := v_{t,\bm{\theta}}(\bm{x}_t)$, and the integrator advances on it without further information. The remainder of this section constructs the missing object: a per-step variance $\sigma_t^2$ that quantifies how trustworthy each velocity estimate is. We refer the reader to Appendix~\ref{appendix:flow_matching} for a self-contained review of flow matching, including the conditional formulation and training objective.

\paragraph{From point estimate to posterior.}
Even though the conditional flow-matching target is uniquely determined, the field it specifies is geometrically delicate (Sec.~\ref{sec:related_work}): symmetry-breaking transitions and finite-data fitting noise make the velocity locally unreliable in characteristic regions, and a single point estimate carries no signal of where it is trustworthy. We therefore treat the network weights $\bm{\omega}$ as random and learn an approximate posterior $q_{\bm{\psi}}(\bm{\omega})$ by maximising the weight-space ELBO
\begin{equation}
\mathrm{ELBO}(\bm{\psi}) = \mathbb{E}_{\bm{\omega} \sim q_{\bm{\psi}}}\!\bigl[\log \mathcal{L}(\mathcal{D} \mid \bm{\omega})\bigr] - \mathrm{KL}\!\bigl[q_{\bm{\psi}}(\bm{\omega}) \,\|\, \pi(\bm{\omega})\bigr],
\label{eq:elbo-param}
\end{equation}
with flow-matching log-likelihood
\begin{equation}
\log \mathcal{L}(\mathcal{D} \mid \bm{\omega}) = -\,\mathbb{E}_{t \sim \mathcal{U}(0,1),\,\bm{x}_t \sim p_t}\!\bigl[\|u_t(\bm{x}_t) - v_{t,\bm{\omega}}(\bm{x}_t)\|^2\bigr].
\label{eq:cfm-likelihood}
\end{equation}
Here $\bm{\psi}$ denotes all variational parameters. Eq.~\eqref{eq:elbo-param} is obtained by inducing $q_{\bm{\psi}}(v_t)$ as the pushforward of $q_{\bm{\psi}}(\bm{\omega})$ through the network parameterization $\bm{\omega} \mapsto v_{t,\bm{\omega}}$, with a matching pushforward prior; we work directly in weight space and defer the formal functional derivation to App.~\ref{sec:vi_vector_fields}.

This formulation gives us two useful properties. First, the simulation-free conditional flow-matching estimator of \citet{lipman2022flow} extends to the variational setting: the gradient of the expected log-likelihood under $q_{\bm{\psi}}$ coincides with the gradient of the standard CFM regression loss evaluated on weight samples (Obs.~\ref{obs:equivalence_grads}), so we inherit a tractable, simulation-free training objective at no extra cost. Second, the posterior-mean velocity field $\mu_t(\bm{x}) := \mathbb{E}_{\bm{\omega}\sim q_{\bm{\psi}}}[v_{t,\bm{\omega}}(\bm{x})]$ still satisfies the continuity equation (Obs.~\ref{prop:bayesian_model_average_vectorfields}), so Bayesian model averaging produces a valid generative flow rather than a heuristic combination of vector fields. Together, these properties mean the variational treatment preserves both the training procedure (up to a KL regularizer) and the generative semantics of standard flow matching, while equipping the model with a posterior over velocities from which a per-sample confidence signal can be derived.

\paragraph{State-dependent posterior via VAD.}
The construction so far leaves the variational family $q_{\bm{\psi}}$ unspecified. A standard mean-field Gaussian would suffice for tractability, but it would assign the same posterior spread everywhere in the trajectory and, by Obs.~\ref{prop:bayesian_model_average_vectorfields}, average over a fixed set of admissible vector fields regardless of where the flow is locally reliable. The pathology we are trying to detect, however, is fundamentally local: critical times where the trajectory's posterior over endpoints contracts onto a single mode are exactly where the velocity becomes most sensitive (Sec.~\ref{sec:related_work}), and these regions are not known in advance. We therefore want a posterior whose spread is itself a function of the state $\bm{x}$ and time $t$ along the trajectory.

We parameterise $q_{\bm{\psi}}$ using variational adaptive dropout (VAD) \citep{coscia2025barnn}, which makes the per-weight perturbation depend on the input. Each weight $\bm{\omega}_l$ in layer $l$ has the Gaussian posterior
\begin{equation}
q_{\bm{\psi}}(\bm{\omega}_l \mid \bm{x},t) = \mathcal{N}\!\bigl(\bm{\theta}_l,\; \alpha_l(\bm{x},t)^2\,\bm{\theta}_l^2\bigr),
\label{eq:vad}
\end{equation}
and sampled via the local-reparametrization trick, see App.~\ref{appendix:vad_intro} for more details on VAD. The dropout scale $\alpha_l(\bm{x},t)$ is the output of a small per-layer inference network $E_{\bm{\gamma}}$; the trainable parameters of FMwC are therefore $\bm{\psi} = \{\bm{\theta}, \bm{\gamma}\}$, with $\bm{\theta}$ the main model weights and $\gamma$ the inference network weights. The role of the inference network is to broaden the posterior at points of geometric ambiguity and tighten it elsewhere, learned without supervision: the only training signal it receives is the ELBO itself, in which the KL term reduces to a closed-form expression in $\alpha_l$ alone (Eq.~\ref{eqn:kl_analytical}) and acts as a regulariser on the dropout rates beyond the standard flow-matching loss.

\subsection{Variance Propagation}
\label{sec:method-propagation}
VAD as introduced in \S\ref{sec:method-fmwc} gives a Gaussian posterior over weights at each linear layer; under the local reparameterisation \citep{kingma2015variational}, this induces a Gaussian pre-activation in closed form. Stacking layers with nonlinearities generally breaks this closed form. Sampling-free moment propagation closes this gap within a single forward pass \citep{wang2013fast, postels2019sampling}; FMwC carries this further by propagating moments through both the network \emph{and} the ODE trajectory in one integration: at every ODE step, we compute the posterior-mean velocity $\mu_t$ together with its variance $\sigma_t^2$.

\paragraph{Through the network.}
Under the local reparameterisation, the pre-activation at each linear layer $l$ is Gaussian with closed-form mean and variance, so the pair\footnote{We drop the time index on internal layer moments for clarity, writing $(\mu_l, \sigma_l^2)$ for the moments at layer $l$ at fixed time $t$. The subscript $t$ is reserved for the network's output, i.e.\ $(\mu_t, \sigma_t^2)$ are the moments at the final layer.} $(\mu_l, \sigma_l^2)$ propagates layer-to-layer without sampling \citep{molchanov2017variational}:
\begin{equation}
\mu_{l+1} = \bm{\theta}_l\,\eqblue{\mu_{l}} + b_l,
\qquad
\eqred{\sigma_{l+1}^2} = \eqorange{\alpha_l}\bigl(\bm{\theta}_l^{\odot 2} \eqblue{\mu_{l}^{\odot 2}}\bigr) + \bigl(\eqorange{\alpha_l} + 1\bigr)\bigl(\bm{\theta}_l^{\odot 2}\eqred{\sigma_{l}^2}\bigr),
\label{eq:layerprop}
\end{equation}
where $(\mu_l, \sigma_l^2)$ denotes the post-activation moments entering layer $l$ and $(\mu_{l+1}, \sigma_{l+1}^2)$ the pre-activation moments leaving it; $\bm{\theta}_l^{\odot 2}$ denotes elementwise square, and {\eqblue{blue}} = mean, {\eqred{red}} = variance, {\eqorange{orange}} = learned noise scale. The recursion implies that variance grows multiplicatively with depth, making multi-layer VAD propagate uncertainty through the network. Finally, pointwise nonlinearities $\phi$ require computing $\mathbb{E}[\phi(z)]$ and $\mathrm{Var}[\phi(z)]$ for $z \sim \mathcal{N}(\mu_l, \sigma_l^2)$; we use $10$-node Gauss--Hermite quadrature \citep{postels2019sampling}, with first- and second-order Taylor linearisation \citep{wang2013fast} as a near-identical alternative at the depths used here (Tab.~\ref{tab:propagation}).

\paragraph{Through the ODE.}
At each Euler step the network emits the pair $(\mu_t, \sigma_t^2)$, where $\mu_t$ is the posterior-mean velocity defined above. Advancing the position along $\bm{x}_{t+\Delta t} = \bm{x}_t + \mu_t\,\Delta t$ therefore integrates the Bayesian model average flow, which Obs.~\ref{prop:bayesian_model_average_vectorfields} guarantees is itself a valid generator. The variance $\sigma_t^2$ is recorded at every step. This is the design choice that distinguishes FMwC from MC-Dropout and Laplace methods: rather than sampling $k$ velocity fields and integrating $k$ trajectories, we integrate the posterior-mean trajectory and track the variance analytically along it, at approximately the same cost as deterministic flow matching. The output of a single forward integration is therefore not a scalar uncertainty but a \emph{trajectory} of velocity variance with non-trivial temporal structure (Fig.~\ref{fig:method_overview}, \emph{left}): variance can fall as the flow commits to a mode, or persist as the flow fails to.

\subsection{Reading the Trajectory}
\label{sec:method-readouts}

For filtering, the variance trajectory $\{\sigma_t^2\}_{t=0}^{T}$ (the full per-step sequence emitted along the integration) must be aggregated into a per-sample scalar that ranks samples for retention; other downstream uses can read the trajectory directly (Sec.~\ref{sec:experiments}). Three label-free aggregations come directly from the trajectory's geometry. \emph{Endpoint confidence} reads off $\sigma_T^2$ at the final step,
\begin{equation}
c_{\mathrm{ep}}(\bm{x}_0) \;=\; \bigl\| \sigma_{T} \bigr\|_2 ;
\label{eq:cep}
\end{equation}
\emph{integrated confidence} sums per-step velocity std along the trajectory, $c_{\mathrm{int}}(\bm{x}_0) = \sum_{t}\bigl\|\sqrt{\sigma_t^2}\bigr\|_2\,\Delta t$; and the \emph{temporal confidence ratio} contrasts late- with early-window variance,
\begin{equation}
c_{\mathrm{tc}}(\bm{x}_0) \;=\;
\frac{\sum_{t\,>\,\tau_\ell}\bigl\|\sqrt{\sigma_t^2}\bigr\|_2}
     {\sum_{t\,<\,\tau_e}\bigl\|\sqrt{\sigma_t^2}\bigr\|_2},
\label{eq:ctc}
\end{equation}
which lifts when converging samples shrink late variance while stuck samples do not. The full sequence of variance along the trajectory is also a useful signal to link to specific downstream outcomes when labels are available: given a small held-out set with an outcome flag (a downstream-quality threshold, a classifier-failure label, a stability target), we can map the sequence to that outcome directly via a learned head $c_\phi(\bm{x}_0) = \phi_{\bm{\eta}}\!\bigl(\sigma_0^2,\dots,\sigma_T^2\bigr)$, fit with L1-regularised logistic regression or gradient boosting; this turns $\{\sigma_t^2\}$ into a task-specific score at little extra inference cost. Cross-method baselines re-cast the same confidence by resampling instead: MC-dropout and deep ensembles draw $k$ trajectories from $\bm{x}_0$ and report the endpoint dispersion $c_{\mathrm{disp}} = \bigl\|\mathrm{std}_k\!\bigl(\bm{x}_1^{(i)}\bigr)\bigr\|_2$, at $k{\times}$ trajectory cost.

\subsection{Where and When Does Uncertainty Arise?}\label{sec:method-intuition}
The construction above is mechanically motivated, but the reason the resulting $\sigma_t^2$ is informative is geometric. The optimal flow-matching velocity is a conditional expectation over endpoints (Sec.~\ref{sec:related_work}), and at critical times this average undergoes spontaneous symmetry-breaking: the endpoint posterior contracts onto a single mode, the field's susceptibility diverges, and the local divergence of the velocity governs how information concentrates \citep{raya2023spontaneous,ambrogioni2024thermodynamics,stancevic2026information}. At these bifurcation points the regression target is an average over multiple still-plausible continuations, so any single-valued prediction is irreducibly uncertain — and this is the regime where finite-capacity ELBO training also has the largest residual error to explain.

The inference network picks up on this without supervision. Maximising the ELBO trades reconstruction error against KL: tighter posteriors (smaller $\alpha_l$) reduce the KL but inflate reconstruction, and the optimum broadens $\alpha_l(\bm{x},t)$ exactly where the velocity is hardest to fit, i.e.\ the bifurcation regions. Eq.~\ref{eq:layerprop} then amplifies this broadening multiplicatively with depth, leaving bifurcation-point inputs with large $\sigma_t^2$. Figure~\ref{fig:method-div-vs-std} shows this on the Checkerboard distribution, where $|\nabla\!\cdot\!v_{t,\bm{\theta}}|$ is available analytically: the single-pass $\sigma_t$ reproduces its structure at the cost of a mere deterministic forward pass. Quantitative correspondence across modalities is reported in Sec.~\ref{sec:exp-mech}.
\begin{figure}[tbb]
  \centering
  \includegraphics[width=\textwidth]{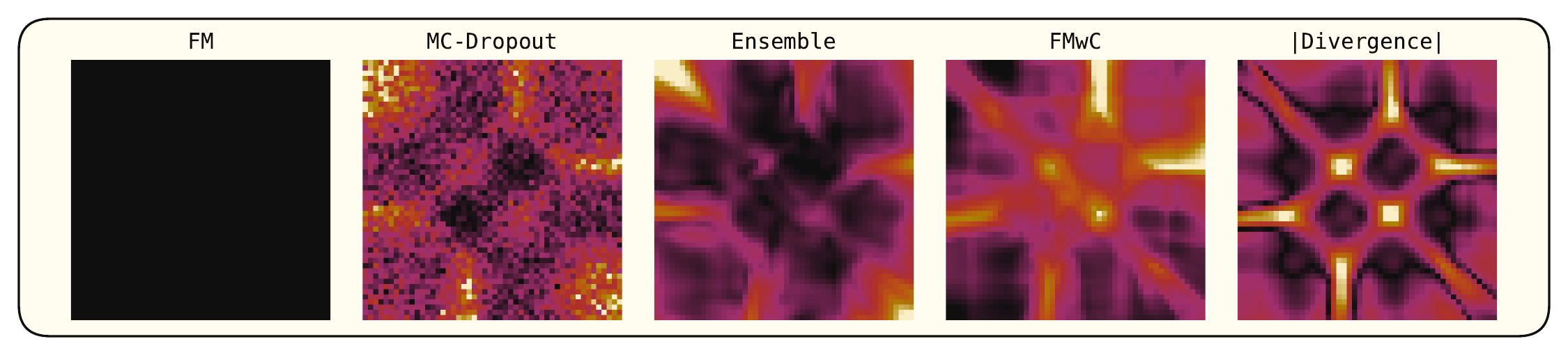}
  \caption{\textbf{$\sigma_t^2$ recovers the divergence structure of the learned velocity field at bifurcation time.} Spatial fields at $t{=}0.6$ on the Checkerboard, where $|\nabla\!\cdot\!v_{t,\bm{\theta}}|$ (right) is available analytically. Quantitative correlation across modalities in Sec.~\ref{sec:exp-mech}.}
  \label{fig:method-div-vs-std}
\end{figure}

%% file: sections/experiments.tex
\section{Experiments}
\label{sec:experiments}

We demonstrate the effectiveness of FMwC by applying it across domains where generative modelling plays a central role. First, as illustrative examples, we apply FMwC to a 2D Checkerboard and to class-conditional MNIST generation under standard Euclidean flow matching. We then turn to a more challenging setting and apply FMwC to de novo inorganic crystal generation under Riemannian flow matching, where it must operate on a non-Euclidean manifold. We first verify that standard generation under FMwC preserves sample quality, then exercise the confidence score in three downstream uses: filtering low-confidence samples, editing trajectories at the timesteps where they lock onto a mode, and adapting ODE step size to where the flow is ambiguous. Finally, we show that $\sigma_t^2$ indeed correlates with the magnitude of the learned velocity field's divergence, with the correlation strengthening as the geometry of the flow becomes richer. For additional details on datasets, model hyperparameters and metrics we refer to App.~\ref{app:exp_details}.

\begin{table}[h!]
\centering
\caption{\textbf{Standard generation under FMwC matches or improves FM on every modality, with negligible inference overhead at scale.} Quality columns: one headline metric per modality (lower is better in all three). Cost columns: forward-pass FLOPs ratio relative to FM, measured. Bold marks the best per column.}
\label{tab:quality-summary}
\resizebox{\textwidth}{!}{\input{assets/tables/table_quality_summary}}
\end{table}

\paragraph{Standard generation under FMwC preserves and improves sample quality.}
Table~\ref{tab:quality-summary} reports sample quality across the three datasets. FMwC matches or improves FM's quality metric on every dataset, and notably improves over FM on Checkerboard and class-conditional MNIST, showing that learning the confidence score can also improve sample quality. Inference cost scales favourably with backbone size. VAD adds a roughly fixed cost per layer, so on the tiny Checkerboard backbone this overhead is significant, but on larger backbones like MNIST and FlowMM the backbone itself dominates and the overhead becomes negligible. Finally, multi-trajectory baselines are absent on Crystals because deep ensembles would require $5{\times}$ retraining of FlowMM and MC-dropout would require $5{\times}$ stochastic trajectories per generated structure, both prohibitive at the throughput LeMat GenBench~\citep{betala2025lemat} evaluates against.

\begin{table}[h!]
\centering
\caption{\textbf{FMwC trajectories carry filtering signal across modalities at single-trajectory cost.} AUPRC ($\uparrow$) of confidence-guided retention on Checkerboard misplacement, MNIST classifier failure, and Crystal $e_{\text{above hull}}<0.1$ eV/atom. The FM row reports the positive-class rate (random-retention floor); meaningful comparison is lift over it. Row tints denote scoring family: \textcolor[HTML]{205EA6}{blue} = unsupervised aggregations of the FMwC variance trajectory, \textcolor[HTML]{66800B}{green} = supervised heads on engineered features.}
\label{tab:filter-auprc}
\resizebox{\textwidth}{!}{\input{assets/tables/table_5_filter_auprc}}
\end{table}

\paragraph{Filtering low-confidence samples reaches ensemble-grade lift at single-trajectory cost.}
Table~\ref{tab:filter-auprc} reports filtering AUPRC on the three modalities. The FM row is the random-retention floor (it equals the positive-class rate); what matters is lift above that floor. Two findings stand out. First, on Checkerboard the temporal aggregation of $\sigma_t^2$ reaches $0.55$ AUPRC from a single trajectory, matching the $k{=}5$ MC-Dropout and Ensemble baselines at one fifth of the inference cost. Second, a small supervised head trained on the variance trajectory is the best single-trajectory filter on every modality, $0.83$ on Checkerboard, $0.09$ on MNIST, $0.06$ on Crystals, still using a single trajectory. The MNIST and Crystals numbers are modest in absolute terms because the targets are hard (a classifier disagreeing with the conditioned digit; an unrelaxed UMA single-point predicting $e_{\text{above hull}}<0.1$ eV/atom), but in every case the variance trajectory carries information the FM baseline does not.

\begin{figure}[h!]
    \centering
    \includegraphics[width=\textwidth]{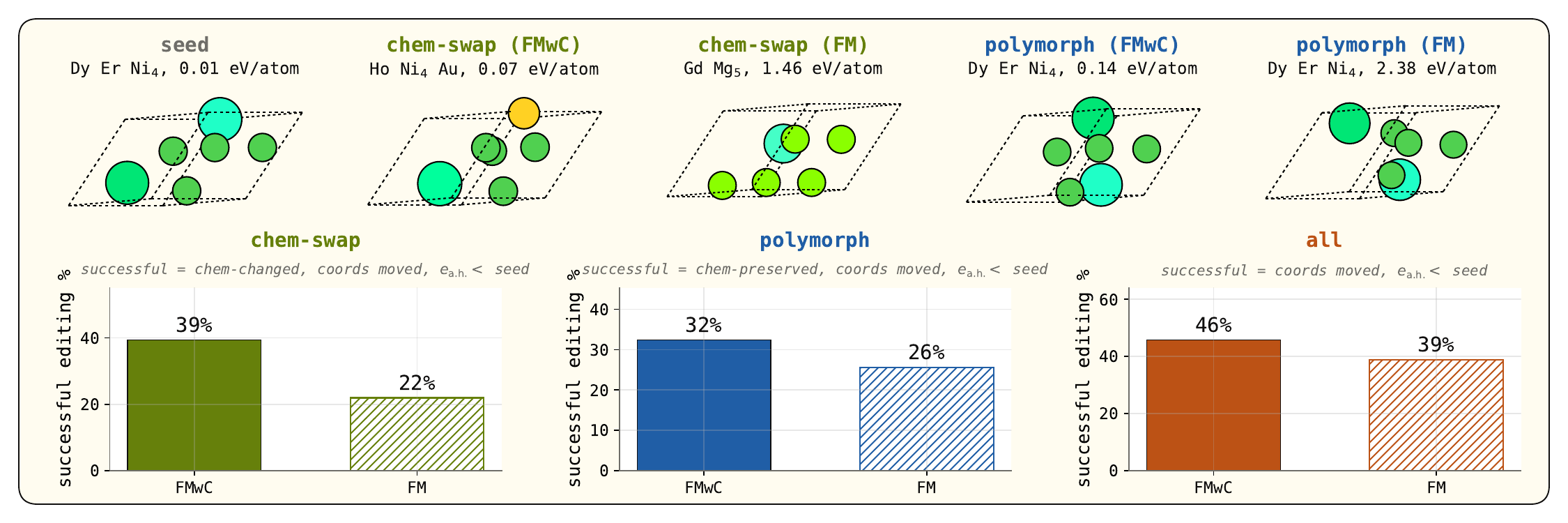}        
    \caption{\textbf{FMwC's per-channel $t^\star$ targets a specific bifurcation; FM at random $t$ misses it.} (a)~Seed Dy Er Ni$_4$ (on the hull) under chem-swap (atoms-only at $t^\star_{\text{atom}}$) and polymorph (coords-only at $t^\star_{\text{coord}}$): FMwC's edits stay near the hull, FM's drift off. (b--d)~Per-mode success rate over 50 seeds $\times$ 20 replicates; FMwC ($t{=}t^\star$, solid) vs FM (uniform random $t$, hatched), at matched channel mask and $\sigma$. Success requires the per-mode chemistry intent (chem-changed for chem-swap, chem-preserved for polymorph, either for all-channels), coord-RMSD $>0.10$, and $e_{\text{a.h.}}$ below the seed's.}
    \label{fig:editing}
\end{figure}

\paragraph{FMwC enables controlled editing of crystals without retraining.}
The variance trajectory exposes \emph{when} the model commits to a structure. To edit a generated sample we (i)~pick a perturbation timestep $t^\star$ and (ii)~inject Gaussian noise into the channel being edited at $t^\star$, then re-integrate the deterministic flow to $t{=}1$ with all channels free to update. The channel is one of atoms, coordinates, or the full latent; FMwC picks $t^\star$ as the argmax of $\sigma_t^2$ in that same channel, while FM picks $t^\star$ uniformly at random. Figure~\ref{fig:editing} reports successful editing rate on FlowMM under three modes that pair an injection channel with a success criterion: \emph{chem-swap} injects atom-channel noise and counts a success when the chemistry changes; \emph{polymorph} injects coord-channel noise and counts a success when the chemistry is preserved; \emph{all-channels} injects into the full latent and counts a success regardless of chemistry. All three additionally require the structure to have moved (periodic boundary aware coord-RMSD $>0.10$) and $e_{\text{a.h.}}$ to fall below the seed's. At matched channel mask and noise budget, FMwC produces $6{-}17$ percentage points more successful edits than FM, with the gap widest on the constraint-respecting modes and narrowest on \emph{all-channels}, the predicted pattern: locating the critical timestep matters most when the edit must respect a constraint. These experiments confirm empirically that the high-$\sigma_t^2$ regions of the trajectory coincide with critical decision points: editing at $t^\star$ reroutes the sample, while editing away from it either dissolves the perturbation (early $t$) or leaves no time to react (late $t$).

\begin{figure}[h!]
  \centering
  \includegraphics[width=\textwidth]{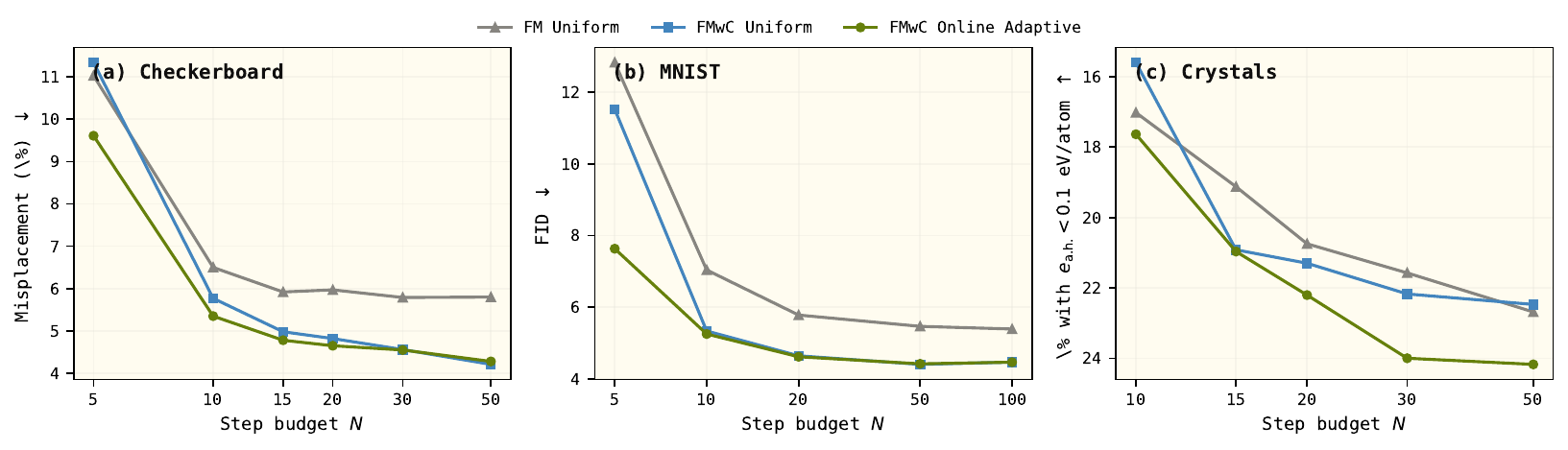}
  \caption{\textbf{A single per-sample adaptive-stepping signal, $\sigma_t^2(t)$, generalises across modalities.} Each panel plots the application's headline quality metric against integrator step budget $N$, oriented so down-and-right is better (Crystal y-axis inverted). Methods: FM Uniform (grey), FMwC Uniform (blue), and FMwC Online Adaptive (green), with the controller chosen per modality.}
  \label{fig:cross-app-adaptive}
\end{figure}

\paragraph{Adaptive stepping reallocates compute to the timesteps where the flow is least decided.}
The variance trajectory suggests an obvious use: allocate integrator steps where the flow is least decided to improve sample quality. We instantiate this with per-sample controllers on $\sigma_t^2$ at a fixed number of forward evaluation (NFE) budget. A naive controller -- an EMA on the trajectory norm of $\sigma_t^2$ -- improves Checkerboard and MNIST but degrades the Riemannian crystal flow, concentrating compute on early-trajectory variance peaks that do not govern crystal quality. We therefore propose FMwC Online Adaptive, a per-component late-only damping rule that allocates extra steps only to the window in which the per-component variance ratio is rising. The full controller comparison is deferred to Appendix~\ref{app:adaptive-controllers}.

Figure~\ref{fig:cross-app-adaptive} shows that Online Adaptive improves over FM Uniform at every step budget across all three modalities. The gains are largest in the low-budget regime where uniform schedules are most wasteful: at $N{=}5$, Online Adaptive reduces Checkerboard misplacement from $11.0\%$ to $9.6\%$ and MNIST FID from $12.7$ to $7.6$. On Crystals, where the geometry is non-Euclidean and the EMA controller fails, Online Adaptive improves post-relaxation Meta\% over FM Uniform at every $N \geq 10$ and matches at $N{=}30$ the quality FM Uniform attains at $N{=}50$.

\begin{figure}[ttb]
  \centering
  \includegraphics[width=\textwidth]{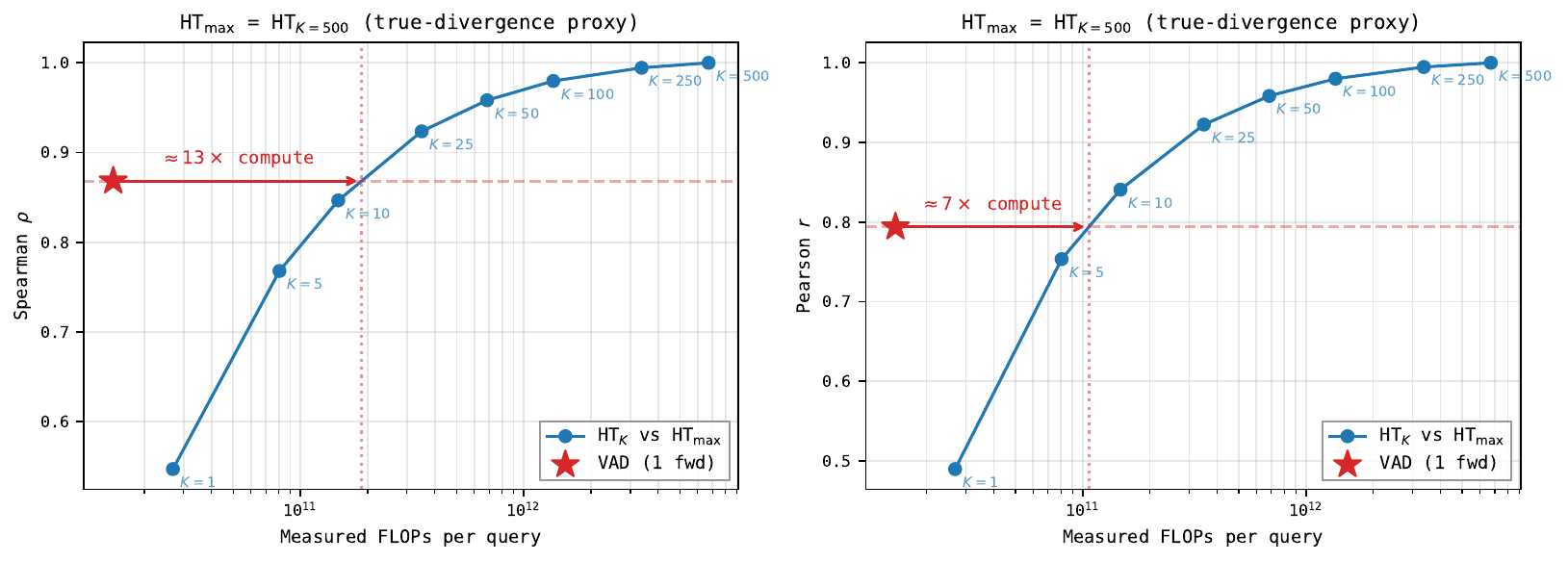}
  \caption{\textbf{On the Crystals model, one forward pass of FMwC matches Hutchinson's divergence.} Spearman  and Pearson correlation against the high-$K$ Hutchinson reference vs.\ measured FLOPs per query, on the FlowMM backbone; sweep over $K \in \{1, \dots, 500\}$.}
  \label{fig:mech-flops}
\end{figure}

\paragraph{$\sigma_t^2$ correlates with the magnitude of the learned-velocity-field divergence on every setting.} Section~\ref{sec:method-intuition} argues that $\sigma_t^2$ should peak at the bifurcation points where the divergence of the learned velocity field also peaks. We test this empirically. Since computing the exact divergence at scale is infeasible, we benchmark against a high-$K$ Hutchinson stochastic-trace estimator of mean divergence and treat it as the reference signal. Our results, summarized in Figure~\ref{fig:mech-flops}, support the prediction in two ways. First, a single forward pass of FMwC matches the rank correlation against the high-$K$ reference that Hutchinson reaches with $K \approx 13$ (Spearman) and $K \approx 7$ (Pearson) probes, the divergence signal is recovered at a fraction of the cost. 

Second, the correspondence is universal and geometry-graded: Table~\ref{tab:universality} reports Pearson and Spearman correlations across all three modalities, in both per-step and integrated regimes. Integrated correlations are uniformly higher, as integration averages out timestep-local noise, and the correlation strengthens monotonically with the geometric richness of the flow, weakest on Checkerboard, stronger on MNIST, strongest on the Riemannian crystal flow.

\begin{table}[h!]
\centering
\caption{\textbf{$\sigma_t^2$ tracks the learned-velocity divergence magnitude $|\nabla\!\cdot\!v_{t,\bm{\theta}}|$ across all three settings.} Pearson $r$ (linear) and Spearman $\rho$ (rank) reported in two regimes: per-step (correlation at each ODE timestep, averaged $\pm$ std) and integrated (correlation between trajectory-summed quantities, one scalar per sample). Single-method by construction: ensembles have no within-model analogue of $\sigma_t^2$.}
\label{tab:universality}
\resizebox{\textwidth}{!}{\input{assets/tables/table_2_overview}}
\end{table}
\label{sec:exp-mech}

%% file: assets/tables/table_quality_summary.tex
\begin{tabular}{lcccccc}
\toprule
& \multicolumn{3}{c}{Quality ($\downarrow$)} & \multicolumn{3}{c}{Inference cost vs FM ($\downarrow$)} \\
\\
Method & Checkerboard mispl \% & MNIST FID & Crystals $\bar{e}_{\text{a.h.}}$ & Checkerboard & MNIST & Crystals \\
\midrule
FM & 5.9 & 5.17 & \textbf{0.21} & $\mathbf{1\times}$ & $\mathbf{1\times}$ & $\mathbf{1\times}$ \\
MC-Dropout ($k$=5) & 7.5 & 4.25 & --- & $5\times$ & $5\times$ & --- \\
Ensemble ($k$=5) & 5.9 & 4.35 & --- & $5\times$ & $5\times$ & --- \\
FMwC & \textbf{4.6} & \textbf{4.22} & \textbf{0.21} & $3.22\times$ & $1.003\times$ & $1.08\times$ \\
\bottomrule
\end{tabular}

%% file: assets/tables/table_5_filter_auprc.tex
\begin{tabular}{lrrrrr}
\toprule
Method & Scoring & Toy AUPRC $\uparrow$ & MNIST AUPRC $\uparrow$ & Crystal AUPRC $\uparrow$ & Traj.\ $\downarrow$ \\
\midrule
FM & random & 0.06 & 0.04 & 0.03 & 1 \\
\rowcolor[HTML]{E1ECEB} MC-Dropout ($k$=5) & endpoint dispersion & \textbf{0.60} & 0.08 & — & 5 \\
\rowcolor[HTML]{E1ECEB} Ensemble ($k$=5) & endpoint dispersion & 0.55 & \textbf{0.11} & — & 5 \\
\rowcolor[HTML]{E1ECEB} FMwC ($k$=5) & endpoint dispersion & 0.58 & 0.04 & \textbf{0.03} & 5 \\
\midrule
\rowcolor[HTML]{E1ECEB} FMwC & endpoint confidence & 0.30 & 0.02 & \textbf{0.04} & 1 \\
\rowcolor[HTML]{E1ECEB} FMwC & integrated confidence & 0.33 & 0.02 & \textbf{0.04} & 1 \\
\rowcolor[HTML]{E1ECEB} \textbf{FMwC} & \textbf{temporal ratio} & \textbf{0.55} & \textbf{0.04} & 0.03 & 1 \\
\midrule
\rowcolor[HTML]{EDEECF} \textbf{FMwC} & \textbf{learned (L1-LR)} & 0.65 & \textbf{0.09} & \textbf{0.06} & 1 \\
\rowcolor[HTML]{EDEECF} \textbf{FMwC} & \textbf{learned (GB)} & \textbf{0.83} & \textbf{0.09} & \textbf{0.06} & 1 \\
\bottomrule
\end{tabular}

%% file: assets/tables/table_2_overview.tex
\begin{tabular}{lrrrr}
\toprule
& \multicolumn{2}{c}{Pearson $r\bigl(\sigma_t^2,\,|\nabla\!\cdot\!v_{t,\bm{\theta}}|\bigr)$ $\uparrow$} & \multicolumn{2}{c}{Spearman $\rho\bigl(\sigma_t^2,\,|\nabla\!\cdot\!v_{t,\bm{\theta}}|\bigr)$ $\uparrow$} \\
\cmidrule(lr){2-3} \cmidrule(lr){4-5}
Application & Per-step (mean $\pm$ std) & Integrated & Per-step (mean $\pm$ std) & Integrated \\
\midrule
Checkerboard & $0.42 \pm 0.26$ & $0.71$ & $0.36 \pm 0.29$ & $0.76$ \\
MNIST & $0.46 \pm 0.18$ & $0.57$ & $0.44 \pm 0.17$ & $0.58$ \\
Crystals & $\mathbf{0.79 \pm 0.09}$ & $\mathbf{0.93}$ & $\mathbf{0.87 \pm 0.04}$ & $\mathbf{0.96}$ \\
\bottomrule
\end{tabular}

%% file: sections/conclusion.tex
\section{Conclusion}
\label{sec:conclusion}
We presented Flow Matching with Confidence (FMwC), which turns flow matching from a generator of samples into a generator of samples paired with a measure of how much to trust them. By learning an input-dependent posterior over the velocity field's weights, FMwC injects noise that regularises training and yields a confidence score at generation time, propagated analytically through the ODE in a single deterministic forward pass at the same cost as standard flow matching. Across 2D density estimation, image generation, and de novo inorganic crystal generation, FMwC preserves baseline sample quality on every modality, confirming that confidence training does not come at the expense of generation fidelity. Beyond the score itself, we show that the variance trajectory unlocks a range of downstream uses at no additional training or inference cost. Filtering by confidence improves sample quality across modalities, the variance peak exposes when the model commits to a structure and enables targeted, constraint-respecting edits without retraining, the variance trajectory drives a per-sample adaptive stepping rule that recovers high quality samples at low step budgets, and it also tracks the magnitude of the learned velocity field's divergence. These results, together with the open questions and limitations discussed in Appendix~\ref{app:limitations}, suggest that FMwC makes confidence-aware generation a practical default for scientific and safety-critical applications, where reliable samples matter as much as plausible ones.

%% file: sections/acknowledgements.tex
\section{Acknowledgements}
We thank Johann Brehmer, Floor Eijkelboom, Fabio Grätz, Pim de Haan, Metod Jazbec, Victor Kawasaki-Borruat, Alessandro De Maria, Govert Verkes for the thoughtful and helpful discussions. We thank the CuspAI team for the supportive research environment and for the engineering platform that this work builds on. This publication is part of the project SIGN with file number VI.Vidi.233.220 of the research programme Vidi, which is (partly) financed by the Dutch Research Council (NWO) under the grant https://doi.org/10.61686/PKQGZ71565.

%% file: sections/appendix.tex
\newpage
\hrule height 4pt
\vskip 0.15in
\vskip -\parskip
\begin{center}
{\Large\sc Flowing with Confidence Appendix \par}
\end{center}
\vskip 0.29in
\vskip -\parskip
\hrule height 1pt
\vskip 0.4in

\startcontents[sections]\vbox{\sc\Large Table of Contents}\vspace{4mm}\hrule height .5pt
\printcontents[sections]{l}{1}{\setcounter{tocdepth}{2}}
\vskip 4mm
\hrule height .5pt
\vskip 10mm
\newpage

\section{Notation and Background}
This section provides a concise review of the relevant background, extends key concepts from prior work, and introduces the notation used throughout the paper.
\subsection{Notation}

We denote by $\boldsymbol{x} \in \mathbb{R}^d$ a data vector or state, with $\boldsymbol{x} \sim p_{\rm data}(\boldsymbol{x})$ drawn from the target empirical data distribution.  

Maps or functions are generally denoted by Latin letters, such as $u$ or $f$, while probabilities and distributions are represented by $p$ or $q$. Neural network parameters are indicated using Greek letters: $\boldsymbol{\theta}, \boldsymbol{\psi}$ for deterministic parameters, and $\boldsymbol{\omega}$ for stochastic parameters. 

Conditional distributions are written explicitly, e.g., $p(\boldsymbol{x} \mid \boldsymbol{y})$ indicating the distribution of $\boldsymbol{x}$ condition on $\boldsymbol{y}$. Learned quantities, time-dependent or indexed quantities, are indicated by subscripts; for instance, $v_{\boldsymbol{\theta}}$ denotes a function $v$ learned with parameters $\boldsymbol{\theta}$, while $v_{t, \boldsymbol{\theta}}$ indicates the function $v_t$ indexed by $t$, with learned parameters $\boldsymbol{\theta}$.  

\subsection{Continuous Transformations for Generative Modeling}\label{appendix:flow_matching}

Generative models based on continuous transformations aim to learn a mapping that transports a simple reference distribution \(p_0\), such as a standard Gaussian, to a complex data distribution \(p_{\rm{data}}\). This is typically achieved by defining a parameterized transformation whose expressiveness is sufficient to capture the structure of the data, while remaining computationally tractable.

A general way to construct such transformations is via \textbf{time-dependent vector fields}. Let  
\begin{equation}
u : \mathbb{R}^d \times [0,1] \rightarrow \mathbb{R}^d
\end{equation} 
be a vector field. This induces a \textbf{flow map}  
\begin{equation}
f: \mathbb{R}^d \times [0,1] \rightarrow \mathbb{R}^d
\end{equation}  
defined as the solution to the \textbf{ordinary differential equation (ODE)}:  
\begin{equation}\label{eqn:push_forward}
\frac{d}{dt} f_t(\boldsymbol{x}_0) = u_t(f_t(\boldsymbol{x}_0)), 
\qquad f_0(\boldsymbol{x}_0) = \boldsymbol{x}_0.
\end{equation}  
Intuitively, \(f_t(\boldsymbol{x}_0)\) tracks the trajectory of a point \(\boldsymbol{x}_0\) as it moves along the vector field from \(t=0\) to \(t=1\). Under suitable smoothness conditions, this flow defines a smooth, invertible mapping that pushes the base distribution \(p_0\) through a family of intermediate distributions \(\{p_t\}_{t \in [0,1]}\), with \(p_1\) approximating the target data distribution.

\subsubsection{Distribution Evolution} 
The evolution of the distributions \(p_t\) along the flow is described by the \textbf{continuity equation}, a partial differential equation (PDE) that relates the vector field to the change in probability density:  
\begin{equation}\label{eqn:continuity}
\frac{\partial p_t(\boldsymbol{x})}{\partial t} + \nabla \cdot \big(p_t(\boldsymbol{x}) u_t(\boldsymbol{x})\big) = 0, 
\qquad t \in [0,1].
\end{equation}  
This PDE guarantees that probability mass is conserved as it flows along the trajectories defined by \(u\). In practice, however, the true transport vector field \(u_t(\boldsymbol{x})\) is unknown and must be approximated. A common approach is to parametrize it using a neural network \(v_{t, \boldsymbol{\theta}}(\boldsymbol{x})\). 

Traditional continuous normalizing flows (CNFs) train this network using a likelihood-based objective, which requires evaluating the \textbf{change-of-variables formula} along the flow:  
\begin{equation}
\log p_1(f_1(\boldsymbol{x}_0)) = \log p_0(\boldsymbol{x}_0) - \int_0^1 \nabla \cdot v_{t,\boldsymbol{\theta}}(f_t(\boldsymbol{x}_0))\, dt,
\end{equation}  
where \(\boldsymbol{x}_1 = f_1(\boldsymbol{x}_0)\). This requires repeatedly solving the ODE during training, which can be computationally expensive and sensitive to numerical errors.

\subsubsection{Flow Matching} Instead of directly modeling densities, Flow Matching~\citep{lipman2022flow} focuses on learning the vector field by matching \textbf{target velocities} at points sampled from the intermediate distributions. Specifically, given initial samples \(\boldsymbol{x}_0 \sim p_0\) and target samples \(\boldsymbol{x}_1 \sim p_1\), one defines a smooth interpolation \(\boldsymbol{x}_t\) for \(t \in [0,1]\). This interpolation specifies a family of time-dependent distributions \(p_t\) and an associated velocity field \(u_t(\boldsymbol{x}_t) = \frac{d \boldsymbol{x}_t}{dt}\) that satisfies the continuity equation.

Even though the true vector field \(u_t(\boldsymbol{x}_t)\) and distributions \(p_t\) are unknown, one can construct a tractable \textbf{per-sample formulation} by defining \textbf{conditional flows} toward individual target datapoints \(\boldsymbol{x}_1\). Concretely, for each \(\boldsymbol{x}_1 \sim p_{\rm data}\), we define a conditional trajectory
\begin{equation}
u_t(\boldsymbol{x} \mid \boldsymbol{x}_1),
\end{equation}
which describes how a sample \(\boldsymbol{x}_0 \sim p_0\) moves toward \(\boldsymbol{x}_1\). The unconditional velocity field can then be expressed as a weighted average over these conditional trajectories:
\begin{equation}
u_t(\boldsymbol{x}) = \int u_t(\boldsymbol{x} \mid \boldsymbol{x}_1) \, 
\frac{p_t(\boldsymbol{x} \mid \boldsymbol{x}_1) \, p_{\rm data}(\boldsymbol{x}_1)}{p_t(\boldsymbol{x})} \, d\boldsymbol{x}_1,
\end{equation}
where \(p_t(\boldsymbol{x} \mid \boldsymbol{x}_1)\) is the conditional distribution of \(\boldsymbol{x}_t\) given the endpoint \(\boldsymbol{x}_1\). 

This formulation allows Flow Matching to define a well-posed regression target for the vector field even without explicit knowledge of the full vector field or intermediate densities. Training then reduces to matching the model velocity \(v_{\boldsymbol{\theta}}(\boldsymbol{x}_t, t)\) to the target conditional velocity along these trajectories, by minimizing:
\begin{equation}\label{eqn:fm_obj}
\mathcal{L}(\boldsymbol{\theta}) 
= \mathbb{E}_{\boldsymbol{x}_1 \sim p_{\rm data}, t \sim \mathcal{U}[0,1],\boldsymbol{x}_t \sim p_t(\boldsymbol{x}_t \mid \boldsymbol{x}_1)} 
\Big[ \big\| v_{t, \boldsymbol{\theta}}(\boldsymbol{x}_t) - u_t(\boldsymbol{x}_t \mid \boldsymbol{x}_1) \big\|^2 \Big].
\end{equation}

\subsection{Bayesian Neural Networks}\label{appendix:bayesian_inference}
Bayesian perspectives have played an important role in the development of deep learning methods, offering a coherent probabilistic framework for representing and reasoning about uncertainty \citep{hinton1993keeping, graves2011practical, blundell2015weight, gal2016dropout}. Instead of treating neural network parameters as fixed, Bayesian approaches consider the weights as random variables drawn from a prior distribution, $\boldsymbol{\omega} \sim p(\boldsymbol{\omega})$. Given a dataset $\mathcal{D} = \{(\boldsymbol{x}_i, \boldsymbol{y}_i)\}_{i=1}^N$ of $N$ independent input–output pairs, the aim is to infer the posterior distribution over the weights, $p(\boldsymbol{\omega} \mid \mathcal{D})$, using Bayes’ theorem \citep{bayes1763lii}. Exact computation of this posterior is generally infeasible for modern deep networks due to the high dimensionality and non-linearities of the parameter space. Consequently, various approximate inference methods have been proposed. 

\subsubsection{Deep Ensembles} Deep ensembles \citep{lakshminarayanan2017simple} approximate the posterior distribution by independently training multiple networks $\boldsymbol{\theta}^{\rm{MAP}}_{i}$ weights given different random initializations. The posterior is then approximated~\citep{wilson2020bayesian} as a Dirac delta $p(\boldsymbol{\omega}\mid\mathcal{D})\approx p_{\rm{ens}}(\boldsymbol{\omega}\mid\mathcal{D})=\sum_{i=1}^N\delta(\boldsymbol{\omega} - \boldsymbol{\theta}^{\rm{MAP}}_{i})$, with $\delta$ the dirac-delta function. While they achieve state-of-the-art performance in many tasks \citep{fort2019deep, tan2023single}, ensembles can be computationally expensive, particularly as the network size or the number of ensemble members grows. Each model in the ensemble requires a full set of parameters, leading to a linear increase in memory usage with the number of networks. This not only increases the storage requirements but also limits the ability to deploy ensembles in memory-constrained environments, such as edge devices or large-scale inference systems. Additionally, training multiple large networks is time-consuming, and inference requires aggregating predictions from all ensemble members, which can significantly slow down real-time applications. As a result, while deep ensembles provide robust uncertainty estimates, their scalability is inherently constrained by both model size and the number of ensemble components.

\subsubsection{Variational-based Methods} Other widely used approaches include Variational Inference-based methods~\citep{gal2016dropout, blundell2015weight, kingma2015variational}, which replace the intractable posterior $p(\boldsymbol{\omega}\mid\mathcal{D})$ with a tractable parametric distribution $q_{\boldsymbol{\psi}}$ that can be optimized efficiently over the parameters $\boldsymbol{\psi}$ to resemble $p(\boldsymbol{\omega}\mid\mathcal{D})$ by minimizing the evidence lower bound (ELBO). Despite the computational efficiency of these methods, they can be notoriously difficult to train, particularly for large-scale networks~\citep{papamarkou2024position}. Challenges include high sensitivity to hyperparameters, difficulties in balancing the trade-off between the likelihood and the KL divergence term in the ELBO~\citep{blundell2015weight}, and issues with poor posterior approximations for complex, high-dimensional weight spaces~\citep{izmailov2021bayesian}. As a result, while variational approaches provide a scalable alternative to ensembles, achieving stable and accurate uncertainty estimates in deep networks remains a non-trivial problem.

\subsubsection{Variational Adaptive Dropout}\label{appendix:vad_intro} \citet{coscia2025barnn} recently proposed Variational Adaptive Dropout (VAD), a Bayesian approach that presents the scalability properties of variational methods, with the accuracy gain typical of ensemble methods.  In VAD, each weight $\boldsymbol{\omega}_i$ in a neural network layer $i$ is modelled with an \emph{input-dependent variational distribution}, typically a Gaussian whose variance scales with the weight magnitude:
\begin{equation}\label{eqn:vad_distribution}
    q_{\boldsymbol{\psi}}(\boldsymbol{\omega}_i \mid \boldsymbol{x}) = \mathcal{N}\bigl(\boldsymbol{\theta}_i, \,\alpha_i(\bm{x})^2 \, \boldsymbol{\theta}_i^2\bigr),
\end{equation}
where $\boldsymbol{\theta}_i$ is the mean parameter, and $\alpha_i(\bm{x})$ is a learned function of the input $\boldsymbol{x}$ controlling the adaptive dropout scale, usually given by an inference network $E_{\boldsymbol{\gamma}}$. Sampling can be reparameterized as
\begin{equation}
    \bm{\omega}_i = \bm{\theta}_i + \alpha_i(\bm{x}) \, |\bm{\theta}_i| \, \bm{\epsilon}_i, \quad \bm{\epsilon}_i \sim \mathcal{N}(0,1),
\end{equation}
which allows low-variance gradient estimates via the \emph{local reparameterization trick}~\citep{kingma2015variational}. For a linear layer with input $\boldsymbol{x}$, the resulting activation $z_j$ is distributed as
\begin{equation}\label{eqn:local-reparam}
    z_j \sim \mathcal{N}\Biggl(\sum_i x_i \theta_{ij}, \sum_i x_i^2 \alpha_i(x)^2 \theta_{ij}^2 \Biggr),
\end{equation}
reducing the cost of sampling from per-weight to per-activation. 
Finally, adopting a zero-mean Gaussian prior with hyperparameter $p\in[0,1]$:
\begin{equation}
    \pi(\bm{\omega}) = \mathcal{N}\left(\bm{0},\, \frac{p}{1-p} \bm{\theta}^2\right),
\end{equation}
yields a closed-form KL divergence that depends only on the VAD coefficients:
\begin{equation}\label{eqn:kl_analytical}
    {\rm{KL}}\big[q_{\bm{\phi}}(\bm{\omega} \mid \bm{x})\,\|\, \pi(\bm{\omega})\big] = {\rm{KL}}(\{\alpha_i\}_i)=\sum_{i=1}^L\frac{(\alpha_i+1)(1-p)}{p}+\log\left({\frac{p}{1-p}}\right)-\log(\alpha_i) - 1.
\end{equation}

This approach scales efficiently because the local reparameterization trick avoids sampling each weight individually, and the adaptive dropout coefficients $\alpha_i(x)$ can be predicted by a small auxiliary network $E_{\boldsymbol{\gamma}}$, enabling input-dependent uncertainty without significantly increasing computational cost.

Variational Adaptive Dropout has been successfully applied to a variety of domains, including autoregressive and recurrent neural networks for molecule generation and PDE solving~\cite{coscia2025barnn}, as well as graph neural networks for machine-learned interatomic potentials (MLIPs)~\cite{coscia2025blips}. In this work, we employ Variational Adaptive Dropout extensively to parametrize the flow-matching vector field, enabling an efficient Bayesian flow-matching model.

\section{Theoretical Results}\label{sec:vi_vector_fields}  
We begin by introducing a prior $\pi(v_t)$ over vector fields $v_t:[0,1]\times\mathbb{R}^d \to \mathbb{R}^d$. Given $\mathcal{D} = \{p_0, p_1\}$, our goal is to find the posterior $\pi(v_t \mid \mathcal{D})$.  
To do so, let us introduce a variational approximation $q_{\bm{\psi}}(v_t)$, and minimize the Kullback-Leibler divergence between the variational and true posterior:
\begin{equation}
\mathrm{KL}\bigl[q_{\bm{\psi}}(v_t)\,\|\,\pi(v_t \mid \mathcal{D})\bigr] = \int\log\frac{dq_{\bm{\psi}}}{d\pi}(v_t)dq_{\bm{\psi}}(v_t),
\end{equation}
where $\frac{dq_{\bm{\psi}}}{d\pi}$ indicates the Radon-Nikodym derivative. Minimizing the divergence above is equivalent to maximizing:
\begin{equation}\label{eqn:elbo_bfm_proposition}
\text{ELBO}(\bm{\psi})=
\mathbb{E}_{v_t \sim q_{\bm{\psi}}} \bigl[ \log \mathcal{L}(\mathcal{D} \mid v_t) \bigr]
- \mathrm{KL}\bigl[q_{\bm{\psi}}(v_t)\,\|\,\pi(v_t)\bigr],
\end{equation}
where $\mathcal{L}(\mathcal{D} \mid v_t)$ is the functional likelihood over probability densities. To build a likelihood which penalizes deviations from a given target vector field that generates the probability path, we write:
\begin{align}
\log \mathcal{L}(\mathcal{D} \mid v_t) 
=- \mathbb{E}_{t \sim U(0,1), \bm{x} \sim p_t(\bm{x})} \Big[ \| u_t(\bm{x}) - v_t(\bm{x}) \|^2 \Big].
\end{align}
Note that the gradient of the log-likelihood is always equivalent to the gradient of the conditional Flow Matching objective~\eqref{eqn:fm_obj}, see ~\citet{lipman2022flow} for reference.

\paragraph{Parameterizing function-space distributions.}  
So far, we have reasoned about a variational posterior over vector fields $v_t$ in an abstract, infinite-dimensional function space. While mathematically elegant, this formulation is not directly implementable: we cannot store or sample arbitrary functions. A natural solution is to \emph{parameterize the vector field using a neural network} $v_{t, \bm{\omega}}(\bm{x})$ with random parameters $\bm{\omega}$. We can then link the variational distribution over the network parameters to the functions directly by:
\begin{equation}\label{eqn:posterior_measure}
q_{\bm{\psi}}(v_t) = \int \delta(v_t - v_{t, \bm{\omega}}) \, q_{\bm{\psi}}(\bm{\omega}) \, d\bm{\omega}.
\end{equation}

Intuitively, this says: to sample a vector field $v_t$, first sample parameters $\bm{\omega} \sim q_{\bm{\psi}}(\bm{\omega})$, then evaluate the corresponding function (network) $v_{t, \bm{\omega}}$ given the random parameters $\bm{\omega}$. Given the parametrization of equation~\eqref{eqn:posterior_measure} and the objective in equation~\eqref{eqn:elbo_bfm_proposition} we obtain the following proposition: 
\begin{proposition}[\textbf{Tractable parameter-space ELBO}]\label{prop:elbo_params}
Let $v_{t,\bm{\omega}}$ be a parametrization of the vector field with parameters $\bm{\omega}$, and let $q_{\bm{\psi}}(\bm{\omega})$ denote a variational posterior over $\bm{\omega}$. Then, given equation~\eqref{eqn:posterior_measure}, the function-space ELBO in equation~\eqref{eqn:elbo_bfm_proposition} reduces to the following tractable parameter-space objective:
\begin{equation}
\text{ELBO}(\bm{\psi}) =
\mathbb{E}_{\bm{\omega} \sim q_{\bm{\psi}}(\bm{\omega})} \bigl[ \log \mathcal{L}(\mathcal{D} \mid \bm{\omega}) \bigr]
- \mathrm{KL}\bigl[q_{\bm{\psi}}(\bm{\omega}) \,\|\, \pi(\bm{\omega}) \bigr],
\end{equation}
where the likelihood evaluates the discrepancy between the true and parametrized vector fields:
\begin{equation}
\log \mathcal{L}(\mathcal{D} \mid \bm{\omega}) 
= - \mathbb{E}_{t \sim U(0,1),\, \bm{x} \sim p_t(\bm{x})} \Big[ \| u_t(\bm{x}) - v_{t,\bm{\omega}}(\bm{x}) \|^2 \Big].
\end{equation}
\end{proposition}

\begin{proof}
Let $\mathcal{V}$ denote the function space of admissible vector fields. By definition, the variational distribution over vector fields is expressed as
\begin{equation}
q_{\bm{\psi}}(v_t) = \int \delta\bigl(v_t - v_{t,\bm{\omega}}\bigr)\, q_{\bm{\psi}}(\bm{\omega}) \, d\bm{\omega},
\end{equation}
where $\delta$ is the Dirac delta functional in the function space $\mathcal{V}$, and $\bm{\omega}\in\Omega$ parameter space.  Let $T:\Omega\to\mathcal{V}$ be the measurable map defined by
\begin{equation}
T(\bm{\omega}) := v_{t,\bm{\omega}} .
\end{equation}
By construction, the variational distribution over vector fields is the push-forward measure
\begin{equation}
q_{\bm{\psi}}(v_t) = T_{\#} q_{\bm{\psi}}(\bm{\omega}).
\end{equation}

Consider first the expected log-likelihood term. By the definition of the push-forward measure:
\begin{equation}
\int_{\mathcal{V}} \log \mathcal{L}(\mathcal{D}\mid v_t)\, dq_{\bm{\psi}}(v_t)
= \int_{\Omega} \log \mathcal{L}(\mathcal{D}\mid T(\bm{\omega}))\, dq_{\bm{\psi}}(\bm{\omega})
= \mathbb{E}_{\bm{\omega}\sim q_{\bm{\psi}}(\bm{\omega})}
\bigl[\log \mathcal{L}(\mathcal{D}\mid \bm{\omega})\bigr].
\end{equation}

Next, assume that the prior over vector fields $\pi(v_t)$ is also induced by a prior $\pi(\bm{\omega})$ on parameters through the same mapping $T$, i.e.\ $\pi(v_t)=T_{\#}\pi(\bm{\omega})$. Since $T$ is deterministic and measurable, the data-processing inequality for the Kullback--Leibler divergence~\citep{cover1991network} under pushforward gives
\begin{equation}
\mathrm{KL}\bigl[q_{\bm{\psi}}(v_t)\,\|\,\pi(v_t)\bigr]
=
\mathrm{KL}\bigl[T_{\#}q_{\bm{\psi}}(\bm{\omega})\,\|\,T_{\#}\pi(\bm{\omega})\bigr]
\;\leq\;
\mathrm{KL}\bigl[q_{\bm{\psi}}(\bm{\omega})\,\|\,\pi(\bm{\omega})\bigr],
\end{equation}
with equality when $T$ is injective. Optimizing the parameter-space objective therefore maximizes a valid lower bound on the functional ELBO, which is tight under injectivity of the parameterization and otherwise tight up to the usual weight-space symmetries (permutations and scalings) of $v_{t,\bm{\omega}}$.

Combining the two terms, the functional-space ELBO is lower-bounded by
\begin{equation}
\text{ELBO}(\bm{\psi})
\;\geq\;
\mathbb{E}_{\bm{\omega}\sim q_{\bm{\psi}}(\bm{\omega})}
\bigl[ \log \mathcal{L}(\mathcal{D} \mid \bm{\omega}) \bigr]
-
\mathrm{KL}\bigl[q_{\bm{\psi}}(\bm{\omega}) \,\|\, \pi(\bm{\omega})\bigr],
\end{equation}
which is the tractable parameter-space objective we optimise.
\end{proof}

\begin{observation}[\textbf{Bayesian model averaging preserves valid flows}]\label{prop:bayesian_model_average_vectorfields}
Let $q_{\bm{\psi}}(v_t)$ be a posterior distribution over vector fields $v_t$ such that each sample $v_t$ satisfies the continuity equation
\begin{equation}
\partial_t p_t + \nabla \cdot \big(p_t v_t \big) = 0.
\end{equation}
Then, the posterior mean vector field
\begin{equation}
\bar{v}_t(\bm{x}) := \mathbb{E}_{v_t \sim q_{\bm{\psi}}}[v_t(\bm{x})]
\end{equation}
also satisfies the continuity equation, i.e.
\begin{equation}
\partial_t p_t + \nabla \cdot \big(p_t \bar{v}_t \big) = 0.
\end{equation}
\end{observation}

\begin{proof}
By assumption, for each $v_t$ in the support of $q_{\bm{\psi}}$, the continuity equation $\partial_t p_t + \nabla\cdot(p_t v_t) = 0$ holds pointwise in $(t,\bm{x})$. Integrating against $q_{\bm{\psi}}$ and exchanging integration with $\partial_t$ and $\nabla\cdot$ by linearity yields
\begin{align}
0 &= \int \big(\partial_t p_t + \nabla\cdot(p_t v_t)\big)\, dq_{\bm{\psi}}(v_t)
= \partial_t p_t + \nabla\cdot\!\left(p_t \int v_t\, dq_{\bm{\psi}}(v_t)\right)
= \partial_t p_t + \nabla\cdot(p_t\, \bar{v}_t).
\end{align}
Bayesian averaging of valid vector fields therefore preserves the validity of the probability flow.
\end{proof}

\begin{observation}[\textbf{Equivalence of Expected Likelihood and Flow Matching Gradients}]\label{obs:equivalence_grads}  
Given a reparameterization of $q_{\bm{\psi}}(\bm{\omega})$ such that $\bm{\omega} = g_{\bm{\psi}}(\bm{\epsilon})$ with $\bm{\epsilon} \sim p(\bm{\epsilon})$, the Bayesian expected likelihood gradient satisfies
\begin{equation}
\begin{aligned}
&-\nabla_{\bm{\psi}} \mathbb{E}_{\bm{\omega}\sim q_{\bm{\psi}}(\bm{\omega})}
\bigl[ \log \mathcal{L}(\mathcal{D} \mid \bm{\omega}) \bigr] \\
&\qquad =
\nabla_{\bm{\psi}} \mathbb{E}_{\bm{\omega}\sim q_{\bm{\psi}}(\bm{\omega}),\, \boldsymbol{x}_1 \sim p_{\rm data},\, t \sim \mathcal{U}[0,1],\, \boldsymbol{x}_t \sim p_t(\boldsymbol{x} \mid \boldsymbol{x}_1)}
\Big[ \big\| v_{t, \bm{\omega}}(\boldsymbol{x}_t) - u_t(\boldsymbol{x}_t \mid \boldsymbol{x}_1) \big\|^2 \Big].
\end{aligned}
\end{equation}
\end{observation}

\begin{proof}
We proceed in two steps:

\textbf{Step 1: Deterministic equivalence.}  
For a deterministic parameter $ \bm{\theta}$, ~\citet{lipman2022flow} showed:
\begin{equation}
\begin{aligned}
&\nabla_{\bm{\theta}} \mathbb{E}_{t \sim U(0,1),\, \bm{x} \sim p_t(\bm{x})} \big[ \| u_t(\bm{x}) - v_{t, \bm{\theta}}(\bm{x}) \|^2 \big] \\
&\qquad =
\nabla_{\bm{\theta}} \mathbb{E}_{\bm{x}_1 \sim p_{\rm data},\, t \sim U(0,1),\, \bm{x}_t \sim p_t(\bm{x} \mid \bm{x}_1)} \big[ \| v_{t, \bm{\theta}}(\bm{x}_t) - u_t(\bm{x}_t \mid \bm{x}_1) \|^2 \big].
\end{aligned}
\end{equation}
This establishes equivalence in the deterministic case.

\textbf{Step 2: Bayesian lifting via reparameterization.}  
Suppose $\bm{\omega} \sim q_{\bm{\psi}}(\bm{\omega})$ admits a reparameterization $\bm{\omega} = g_{\bm{\psi}}(\bm{\epsilon})$ with $\bm{\epsilon} \sim p(\bm{\epsilon})$. Then, for any function $f(\bm{\omega})$, the reparametrization trick~\citep{kingma2014auto} holds:
\begin{equation}
\mathbb{E}_{\bm{\omega} \sim q_{\bm{\psi}}(\bm{\omega})}[f(\bm{\omega})] = \mathbb{E}_{\bm{\epsilon} \sim p(\bm{\epsilon})}[f(g_{\bm{\psi}}(\bm{\epsilon}))],
\qquad
\nabla_{\bm{\psi}} \mathbb{E}_{\bm{\omega} \sim q_{\bm{\psi}}(\bm{\omega})}[f(\bm{\omega})] = \mathbb{E}_{\bm{\epsilon} \sim p(\bm{\epsilon})}[\nabla_{\bm{\psi}} f(g_{\bm{\psi}}(\bm{\epsilon}))].
\end{equation}

Applying this to the deterministic equivalence in Step 1, we obtain
\begin{equation}
\begin{split}
\nabla_{\bm{\psi}} \mathbb{E}_{\bm{\omega} \sim q_{\bm{\psi}}(\bm{\omega})} \mathbb{E}_{t, \bm{x}} \big[ \| u_t(\bm{x}) - v_{t,\bm{\omega}}(\bm{x}) \|^2 \big] 
&= \mathbb{E}_{\bm{\epsilon} \sim p(\bm{\epsilon})} \nabla_{\bm{\psi}} \mathbb{E}_{t, \bm{x}} \big[ \| u_t(\bm{x}) - v_{t,g_{\bm{\psi}}(\bm{\epsilon})}(\bm{x}) \|^2 \big] \\
&= \mathbb{E}_{\bm{\epsilon} \sim p(\bm{\epsilon})} \nabla_{\bm{\psi}} \mathbb{E}_{\bm{x}_1, t, \bm{x}_t \mid \bm{x}_1} \big[ \| v_{t,g_{\bm{\psi}}(\bm{\epsilon})}(\bm{x}_t) - u_t(\bm{x}_t \mid \bm{x}_1) \|^2 \big] \\
&= \nabla_{\bm{\psi}} \mathbb{E}_{\bm{\omega} \sim q_{\bm{\psi}}(\bm{\omega}), \bm{x}_1, t, \bm{x}_t} \big[ \| v_{t,\bm{\omega}}(\bm{x}_t) - u_t(\bm{x}_t \mid \bm{x}_1) \|^2 \big].
\end{split}
\end{equation}
Identifying the left-hand side with the negative expected log-likelihood gradient under $q_{\bm{\psi}}$, we conclude the desired equivalence.
\end{proof}

\paragraph{Variational Adaptive Dropout Flow Matching Objective.}\label{appendix:vad_reparam_fm}
Building on Observation~\ref{obs:equivalence_grads}, we can instantiate the reparameterization
\(\bm{\omega} = g_{\bm{\psi}}(\bm{\epsilon})\) with Variational Adaptive Dropout. In VAD, each weight $\bm{\omega}_i$ is modeled as an input-dependent Gaussian:
\begin{equation}
q_{\bm{\psi}}(\bm{\omega}_i \mid \bm{x}) = \mathcal{N}\Big(\bm{\theta}_i, \, \alpha_i(\bm{x})^2 \, \bm{\theta}_i^2 \Big),
\end{equation}

where $\bm{\theta}_i$ is the mean parameter and $\alpha_i(\bm{x})$ is predicted by a small auxiliary network $E_{\boldsymbol{\gamma}}(\bm{x})$ that controls the adaptive dropout scale. This admits a simple reparameterization:

\begin{equation}\label{eqn:vad_reparam}
\bm{\omega}_i = g_{\bm{\psi}}(\bm{x}, \bm{\epsilon}_i) = \bm{\theta}_i + \alpha_i(\bm{x}) \, |\bm{\theta}_i| \, \epsilon_i, 
\quad \epsilon_i \sim \mathcal{N}(0,1).
\end{equation}

Substituting the VAD reparameterization into the Bayesian Flow Matching gradient, we obtain the training objective:

\begin{equation}\label{eqn:vad_flow_matching}
\text{ELBO}(\bm{\psi}) =
-\mathbb{E}_{\bm{x}_1, t, \bm{x}_t, \bm{\epsilon}}
\Big[
\big\| v_{t, g_{\bm{\psi}}(\bm{x}_1, \bm{\epsilon})}(\bm{x}_t) - u_t(\bm{x}_t \mid \bm{x}_1) \big\|^2
\Big] - \rm{KL}(\{\alpha_i\}_{i\geq 1}).
\end{equation}
With, $\bm{x}_1 \sim p_{\rm data}, \, t \sim \mathcal{U}[0,1], \, \bm{x}_t \sim p_t(\bm{x} \mid \bm{x}_1), \, \bm{\epsilon} \sim \mathcal{N}(0, I)$.

Here:  

\begin{itemize}
\item \(g_{\bm{\psi}}(\bm{x}_1, \bm{\epsilon})\) implements the VAD reparameterization \eqref{eqn:vad_reparam}, with input-dependent adaptive dropout scales \(\alpha_i(\bm{x}_1)\).  
\item Sampling from \(\bm{\epsilon}\) implements the stochasticity of the variational posterior \(q_{\bm{\psi}}\).  
\item The Kullback Leibler term only depends on the adaptive dropout coefficients, as implemented in~\citet{coscia2025blips}.
\end{itemize}

\section{Experimental Details and Hyperparameters}\label{app:exp_details}
In this section we report datasets, evaluation metrics, and architectural hyperparameters used for the main-text experiments. All models were trained on NVIDIA L4 GPUs (24GB), using the validation set for checkpoint selection and the test set for evaluation.

\subsection{Datasets and Metrics}\label{app:dataset-metrics}

We evaluate on three datasets spanning a 2D synthetic target, an image domain, and a structured scientific domain. Reported numbers are averaged over three random seeds unless stated otherwise.

\paragraph{Checkerboard.}                                                                                                                         
  A 2D synthetic target whose support is the eight black cells of a $4{\times}4$ chessboard tiling of $[-2,2]^2$, with uniform density on the
  support and zero outside. The base distribution is $p_0 = \mathcal{N}(0, I_2)$. We train on $5\!\times\!10^5$ samples from the analytical target  
  and evaluate on $10^4$ generated samples. We report three metrics. \emph{Misplacement} is the fraction of generated samples falling outside the
  eight-cell support, evaluated by direct geometric containment (lower is better; the random baseline is $50\%$). \emph{KL} is the symmetric        
  Kullback--Leibler divergence between a Gaussian-kernel KDE of the generated samples (bandwidth chosen by Silverman's rule) and the analytical
  target on a $200\!\times\!200$ grid (lower is better). \emph{Filtering AUPRC} ($\uparrow$) evaluates confidence-guided retention as a binary
  classification: per-sample confidence is the predictor, and a binary label flagging samples we want to discard is the target;
  we sweep a retention threshold from "keep nothing'' to "keep everything'' and integrate precision against recall on the positive class. AUPRC is
  bounded above by $1$ and below by the positive-class prevalence $\pi$ (the precision of random retention).

  \paragraph{MNIST.}
  We use the standard MNIST training split as the target distribution and the held-out test split for downstream classifier evaluation. Images are
  normalised to $[0,1]$ and treated as $784$-dimensional Euclidean vectors; the base distribution is $p_0 = \mathcal{N}(0, I_{784})$. We train on   
  the $6\!\times\!10^4$ training images and evaluate on $10^4$ generated images per run. We report two metrics. \emph{Classifier-derived FID}
  ($\downarrow$) is the Fr\'echet distance between the penultimate-layer activations of a LeNet-style classifier pretrained on MNIST, evaluated on  
  the generated batch and on a held-out reference batch from the test split. \emph{Filtering AUPRC} ($\uparrow$, defined as in
  \textbf{Checkerboard}) uses \emph{classifier failure} as the positive class: the pretrained classifier's top-1 prediction on the generated image
  disagreeing with the requested class label.

  \paragraph{Crystal.}
  We use the FlowMM~\citep{miller2024flowmm} backbone for de novo inorganic crystal generation, trained on the Materials
  Project~\citep{jain2013materials} subset of stable structures used in the original FlowMM release. The flow operates on the Riemannian product    
  manifold of fractional coordinates ($[0,1]^{3N}$ with periodic identifications), lattice parameters ($\mathbb{R}^6$ via the Niggli-reduced
  parameterisation), and atomic species (categorical, embedded in the simplex). We evaluate $2500$ generated structures per run. Quality is reported
   on \emph{LeMat-GenBench}~\citep{betala2025lemat} after post-relaxation by a UMA~\citep{wood2025family} single-point evaluation on the raw CIF:
  \emph{Meta\%} ($\uparrow$) is the fraction of generated structures with $e_{\text{above hull}} < 0.1$~eV/atom (metastable), \emph{Stable\%}
  ($\uparrow$) the fraction at $e_{\text{above hull}} < 0$~eV/atom, and $\bar{e}_{\text{above hull}}$ ($\downarrow$) the mean energy above the
  convex hull in eV/atom. \emph{Filtering AUPRC} ($\uparrow$, defined under \textbf{Checkerboard}) uses \emph{generation success} as the positive
  class: $e_{\text{above hull}} < 0.1$~eV/atom. Due to the computational complexity only a single seed run was reported.

\subsection{Model Architectures and Hyperparameters}
\label{app:architectures}

\paragraph{Checkerboard.}
The deterministic backbone is an MLP with SiLU activations and a sinusoidal time embedding concatenated to the input at every layer; widths and depths are listed in Tab.~\ref{tab:network-hyperparams}. FMwC attaches an inference network $E_{\bm{\gamma}}$ to every linear layer of the backbone. Each $E_{\bm{\gamma}}$ is a small two-layer MLP that takes the layer's mean activation $\mu_l$ and the time $t$ as input and outputs the per-channel log-dropout scale $\log\alpha_l(\bm{x},t)$, which parameterises the Gaussian posterior in Eq.~\ref{eq:vad}; the inference networks share their architecture across layers but not their weights. The baselines reuse the same backbone: \emph{FM} drops the inference networks entirely, \emph{MC-Dropout} replaces them with standard fixed dropout at rate $p$, and \emph{Ensemble} is $k$ independently trained copies of the FM backbone with different random seeds.

\paragraph{MNIST.}
The MNIST backbone is a UNet with sinusoidal time embedding broadcast to every residual block; the base channel count and class-embedding dimension are listed in Tab.~\ref{tab:network-hyperparams}. Class conditioning, when used, is provided as a learned class-embedding vector added to the time embedding before the residual broadcast. Inference networks $E_{\bm{\gamma}}$ are attached to every convolutional and linear layer of the UNet (downsampling stages, bottleneck, upsampling stages), with the same two-layer-MLP construction used on Checkerboard; each $E_{\bm{\gamma}}$ takes the spatially-pooled mean activation of its layer together with the time embedding as input and emits a per-channel log-dropout scale, so VAD noise is shared across spatial positions within a feature map but varies between channels. Baselines are constructed as in Checkerboard, with $k{=}5$ independently retrained UNet copies for the Ensemble.

\paragraph{Crystal.}
The deterministic backbone is the FlowMM equivariant graph network of \citet{miller2024flowmm}, used unchanged from the original architecture; we refer the reader to that paper for the message-passing structure, the manifold-aware update rules, and the parameter count. FMwC attaches an inference network to every linear layer in the message-passing and readout blocks, with the same two-layer-MLP construction used on the Euclidean applications. The deterministic-FM baseline is the FlowMM checkpoint released with~\citet{miller2024flowmm}, retrained under our infrastructure but with the same hyperparameters. Disagreement-based baselines (MC-Dropout, Ensemble) are not reported on the Crystal column because $k{\times}$ retraining of FlowMM and $k{\times}$ stochastic trajectories per generated structure are both prohibitive at the throughput LeMat-GenBench~\citep{betala2025lemat} evaluates against.

\begin{table}[h]
\centering
\caption{\textbf{Network architecture hyperparameters across the three applications.} FMwC attaches an inference network $E_{\bm{\gamma}}$ to every linear layer of the backbone. MC-Dropout shares the backbone and replaces the inference networks with fixed dropout at rate $p$; Ensemble is $k$ independent copies of the FM backbone. Dashes in the Configuration row mark a backbone used unchanged from~\citet{miller2024flowmm}; dashes in the MC-Dropout and Ensemble rows mark baselines not reported on the Crystal column.}
\label{tab:network-hyperparams}
\small
\begin{tabular}{lccc}
\toprule
                   & Checkerboard & MNIST & Crystal \\
\midrule
Backbone           & MLP & UNet & CSPNet \\
Configuration      & hidden 512, depth 3 & base ch.\ 32, cond.\ dim 128 & hidden 512, time-emb.\ 256, depth 6 \\
Inf.\ net width    & 64 & 64 & 64 \\
Inf.\ net depth    & 2  & 2  & 2  \\
MC-Dropout $p$     & 0.1 & 0.1 & --- \\
Ensemble $k$       & 5 & 5 & --- \\
\bottomrule
\end{tabular}
\end{table}

\begin{table}[h]
\centering
\caption{\textbf{Training hyperparameters across the three applications.} All FMwC runs optimise the parameter-space ELBO of Eq.~\ref{eq:elbo-param} with KL weight $\beta$ scaling the closed-form per-layer KL between the VAD posterior $\mathcal{N}(\bm{\theta}_l,\alpha_l^2\bm{\theta}_l^2)$ and a Gaussian prior at fixed rate $\alpha_p$; the KL term is annealed cyclically with a linear warmup. Dashes mark settings not used on the application.}
\label{tab:training-hyperparams}
\begin{tabular}{lccc}
\toprule
                           & Checkerboard & MNIST & Crystal \\
\midrule
Optimiser                  & Adam & Adam & AdamW \\
Learning rate              & $10^{-3}$ & $10^{-4}$ & $5\!\times\!10^{-4}$ \\
Weight decay               & --- & --- & $5\!\times\!10^{-3}$ \\
Batch size                 & 4096 & 256 & 256 \\
Training length            & 20{,}000 iter.\ & 50{,}000 iter.\ & 2{,}000 epochs \\
LR schedule                & --- & --- & cosine to $10^{-5}$ \\
EMA decay                  & --- & --- & 0.999 \\
KL weight $\beta$          & $10^{-3}$ & $10^{-3}$ & $10^{-3}$ \\
Prior rate $\alpha_p$      & 0.3 & 0.3 & 0.1 \\
KL cycles                  & 4 & 4 & 4 \\
KL warmup fraction         & 0.1 & 0.1 & 0.1 \\
\bottomrule
\end{tabular}
\end{table}

\section{Additional Results}
\label{app:additional-results}

\subsection{Single-Layer vs Multi-Layer FMwC}
\label{app:ablation-quality}
We compare two FMwC parameterisations on the Checkerboard. \emph{Output-only} FMwC injects VAD noise into the final linear layer alone. \emph{Multi-layer} FMwC, the version used in the main text, attaches VAD to every linear layer along the depth. Tab.~\ref{tab:singlevsmulti} reports misplacement (fraction of generated points outside the eight-cell support) and KL between a kernel-density estimate of generated samples and the analytical target. The two parameterisations are within $0.3$ percentage points of each other on misplacement, so output-only FMwC is not a quality regression. The cost of the simplification is on the confidence side: only multi-layer propagation accumulates the variance signal that the temporal-ratio score relies on (Tab.~\ref{tab:singlevsmulti-auprc}).

\begin{table}[h]
\centering
\caption{\textbf{Output-only FMwC matches multi-layer FMwC on Checkerboard quality.} Misplacement and KL on the Checkerboard, both lower-is-better, under MAP integration with $N{=}50$ Euler steps. The AUPRC companion that breaks the tie is Tab.~\ref{tab:singlevsmulti-auprc}.}
\label{tab:singlevsmulti}
\input{assets/tables/table_a1_singlevsmulti}
\end{table}

\subsection{Sampling Strategy Ablation}
\label{app:sampling-strategy}
The variational posterior over weights admits three natural decoders for sample generation. (1) \emph{Stochastic}: draw one weight sample $\bm{\omega}\sim q_{\bm{\psi}}(\bm{\omega}\mid\bm{x},t)$ at each ODE step and integrate that single trajectory. (2) \emph{Mean velocity ($k$=5)}: draw $k$ weight samples per step, average their velocity outputs, and integrate the mean. (3) \emph{MAP}: integrate the deterministic mean field $\mu_t(\bm{x})$ produced by removing stochasticity in the network, which is what Eq.~\ref{eq:layerprop} produces alongside the variance trajectory at no extra cost. Tab.~\ref{tab:sampling} reports Checkerboard misplacement under all three for both FMwC parameterisations. The three perform almost exactly, so we adopt MAP throughout the paper: it is the cheapest of the three, and it is exactly the trajectory the variance recursion is already running on.

\begin{table}[h]
\centering
\caption{\textbf{Stochastic, mean-velocity ($k{=}5$), and MAP decoders are within noise on Checkerboard quality.} Misplacement under each decoder for output-only and multi-layer FMwC at $N{=}50$ Euler steps; MAP is the default elsewhere in the paper.}
\label{tab:sampling}
\input{assets/tables/table_a2_sampling_strategy}
\end{table}

\subsection{Full AUPRC Ablation}
\label{app:auprc-full}
Tab.~\ref{tab:singlevsmulti-auprc} is the unified Checkerboard AUPRC table across every method and every readout we tested. Positive class is misplacement (a generated point falling outside the eight-cell support); the FM row is the positive-class rate, which is the random-retention floor. Two facts read off the table. FMwC's temporal-confidence ratio at one trajectory ($\mathrm{AUPRC}{=}0.55$) sits within $0.05$ of the $k{=}5$ FMwC-MC $0.58$), but only integrating a single trajectory.

\begin{table}[t]
\centering
\caption{\textbf{FMwC's temporal-ratio score reaches the $k{=}5$ between-model AUPRC at one-fifth the trajectories on the Checkerboard.} Positive class is misplacement; FM row is the positive-class rate. Trajectories column is forward integrations per generated sample.}
\label{tab:singlevsmulti-auprc}
\input{assets/tables/table_a3_singlevsmulti_auprc}
\end{table}

\subsection{Propagation Variant Comparison}
\label{app:propagation}
The per-layer recursion in Eq.~\ref{eq:layerprop} requires propagating $(\mu,\sigma^2)$ through pointwise nonlinearities $\phi$, which means evaluating $\mathbb{E}[\phi(z)]$ and $\mathrm{Var}[\phi(z)]$ for $z\sim\mathcal{N}(\mu_l,\sigma_l^2)$. We compare three closures: the first-order Taylor expansion of \citet{wang2013fast}, its second-order extension, and $10$-node Gauss--Hermite quadrature \citep{postels2019sampling}. Tab.~\ref{tab:propagation} reports endpoint AUPRC on the Checkerboard for the three. They are identical to two decimal places, so the choice does not matter at the depths used in this paper; we default to Gauss--Hermite as the most numerically stable of the three under wider activations.

\begin{table}[t]
\centering
\caption{\textbf{First-order Taylor, second-order Taylor, and Gauss--Hermite are indistinguishable at the depths used in this paper.} Endpoint AUPRC on the Checkerboard for the three nonlinearity closures inside the layer-to-layer variance recursion of Eq.~\ref{eq:layerprop}.}
\label{tab:propagation}
\input{assets/tables/table_a4_propagation}
\end{table}

\subsection{Endpoint vs Trajectory-Accumulated Scoring}
\label{app:scoring}
The two readouts $c_{\mathrm{ep}}$ and $c_{\mathrm{tc}}$ defined in Sec.~\ref{sec:method-fmwc} differ in what they aggregate: $c_{\mathrm{ep}}$ is the analytical variance at the integration endpoint $t{=}1$, while $c_{\mathrm{tc}}$ is the per-component late-window ratio of the variance trajectory. Tab.~\ref{tab:endpoint-vs-traj} compares the two on the Checkerboard. Trajectory-integrated scoring lifts AUPRC from $0.30$ to $0.33$ for the analytical readout at one trajectory: confidence has temporal dynamics that endpoint-only scoring averages over. The Monte-Carlo variant at $k{=}5$ reaches $0.58$ at the endpoint, which is the upper bound at this trajectory budget but requires five integrations per sample.

\begin{table}[ttb]
\centering
\caption{\textbf{Trajectory-accumulated scoring lifts FMwC's analytical AUPRC by $+0.03$ at one trajectory.} Endpoint vs.\ trajectory-integrated AUPRC on the Checkerboard for the analytical variance estimators.}
\label{tab:endpoint-vs-traj}
\input{assets/tables/table_a5_endpoint_vs_traj}
\end{table}

\subsection{Adaptive Stepping: Controllers}                                                                                                              
  \label{app:adaptive-controllers}                                                                                                                         
  Both controllers run on top of the FMwC variance trajectory at one forward pass per integrator step on a fixed budget of $N$ NFE per sample. Index
  samples by $j$, steps by $i\in\{0,\dots,N{-}1\}$; write $t_j$ for the current per-sample time at step $i$ and $K_j = N - i$ for the steps remaining. The 
  per-sample scalar trajectory norm of $\sigma_t^2$ is modality-specific: $\sigma_j = \bigl(\sum_d \sigma^2_{t,d}\bigr)^{1/2}$ for Checkerboard (sum over
  the two Euclidean dimensions), $\sigma_j = \tfrac{1}{CHW}\sum_{c,h,w}\sqrt{\sigma^2_{t,c,h,w}}$ for MNIST (mean of the pixel-wise std over channels and  
  pixels), and $\sigma_j = \sigma_{\mathrm{coord},j}+\sigma_{\mathrm{lattice},j}+\sigma_{\mathrm{type},j}$ for the crystal flow (coord and type summed over
   atoms, lattice as the Frobenius norm of its $3{\times}3$ variance matrix). Its EMA at rate $\lambda\in(0,1]$ is $\bar\sigma_j \leftarrow
  (1-\lambda)\,\bar\sigma_j + \lambda\,\sigma_j(t_j)$, initialised to $\sigma_j(0)$.
  
  The \emph{naive EMA controller} turns $(\bar\sigma_j,\sigma_j)$ into a
   per-sample step
  $\Delta t_j \leftarrow \tfrac{1-t_j}{K_j}\bigl(\bar\sigma_j/\sigma_j(t_j)\bigr)^{b}$,
  clipped to $[\Delta t_{\min},\,1-t_j]$ with $\Delta t_{\min} = 10^{-2}/N$ and forced to $1-t_j$ on the last step so every sample lands at $t{=}1$ in     
  exactly $N$ NFE; $b\geq 0$ is the boost factor ($b{=}0$ recovers uniform). Because $\bar\sigma_j$ peaks early on every modality, this rule allocates     
  compute to the early window — well-aligned for the Euclidean toys but not for the crystal flow, where the window that governs sample quality is late.    
    
  \emph{FMwC Online Adaptive}, deployed across all three modalities, replaces the scalar trajectory norm with a per-component decomposition and gates the  
  rule on a late-only window. Splitting $\sigma_t^2$ into per-component variances $\sigma^2_{c,j}(t_j)$ with $c$ ranging over the modality's components
  (Euclidean axes for Checkerboard, image channels for MNIST, $\{\mathrm{coord},\mathrm{lattice}\}$ for the crystal flow), and tracking per-component EMAs 
  $\bar\sigma_{c,j}$ (rate $\lambda$, init $\sigma_{c,j}(0)$), the excess ratio               
  $r_{c,j}(t_j) = \max\!\bigl(\sigma_{c,j}(t_j)/\bar\sigma_{c,j} - 1,\,0\bigr)$
  fires only on components currently above their own running mean. The grid stays uniform $\Delta t = 1/N$; for $t_j \geq t_{\mathrm{late}}$ the velocity  
  slice $v_c$ is damped as $v_c \leftarrow \bigl(1 - \kappa_c\,r_{c,j}(t_j)\bigr)_{+}\,v_c$ with $\kappa_c\geq 0$ a per-component damping gain and         
  $(\cdot)_{+}$ the ReLU clip, and is left unchanged for $t_j < t_{\mathrm{late}}$. On the FlowMM crystal flow this composes with the FlowMM anti-annealing
   factor $1 + s_0\, t$ so the effective per-component slope is $s_{c,j} = s_0\,(1 - \kappa_c\, r_{c,j}(t_j))_{+}$ with $s_0$ the base slope shared with FM
   Uniform. The two ingredients that recover the gain the naive scalar EMA loses on crystals are the late-only gate (compute is not spent on the early peak
   in $\bar\sigma_j$) and the per-component split (damping responds to the component that is actually rising at $t_j$).
   
\subsection{Hutchinson $K$-Sweep at Crystal Scale}
\label{app:hutchinson-sweep}
Tab.~\ref{tab:universality} in the main text and Fig.~\ref{fig:mech-flops} both use a high-$K$ Hutchinson stochastic-trace estimator of $|\nabla\!\cdot\!v_{t,\bm{\theta}}|$ as the divergence reference on the FlowMM crystal backbone~\citep{miller2024flowmm}, since the exact divergence requires $D$ backward passes per step at the model dimension and is not directly accessible at this scale. Tab.~\ref{tab:hutchinson-full} reports the full $K$-sweep, $K\in\{1,5,10,25,50,100,250,500\}$, in two regimes: per-step (correlation evaluated at each ODE timestep, then averaged across timesteps) and integrated (correlation between trajectory-summed quantities, one scalar per sample). The trajectory-integrated correlation is essentially flat in $K$ ($\rho$ rises from $0.947$ at $K{=}1$ to $0.955$ from $K{=}5$ onwards), while the per-step correlation continues to climb and only saturates around $K\approx 100$ ($\rho{=}0.853$, vs.\ $0.500$ at $K{=}1$). This is the empirical basis for the matched-FLOPs claim in Sec.~\ref{sec:exp-mech} that one forward pass of FMwC corresponds to Hutchinson at $K\approx 7$ probes on a per-step basis.

\begin{table}[h]
\centering
\caption{\textbf{Per-step Hutchinson correlation climbs to saturation around $K{\approx}100$; trajectory-integrated correlation is saturated from $K{=}5$.} Spearman $\rho$ and Pearson $r$ between Hutchinson at $K$ probes and the high-$K$ reference on the FlowMM crystal backbone, in per-step and integrated regimes.}
\label{tab:hutchinson-full}
\input{assets/tables/table_a8_hutchinson_full}
\end{table}

\section{Limitations and Future Work}\label{app:limitations}

While FMwC delivers calibrated, actionable confidence scores across synthetic, image, and scientific domains, several limitations remain that we outline below to guide future work.

Our Crystal evaluation is reported on a single seed due to the cost of UMA~\citep{wood2025family} post-relaxation and LeMat-GenBench~\citep{betala2025lemat} throughput; trends are consistent with the multi-seed Checkerboard and MNIST results, but a full sweep would tighten variance estimates. More broadly, our benchmarks span a 2D target, an image domain, and a manifold-valued scientific domain, but do not yet probe text, video, or large-scale molecular systems.

On the theoretical side, while our variance decomposition identifies \emph{when} a sample is uncertain, a formal link between the posterior divergence along a trajectory and the resulting sample-level confidence is still missing. Establishing such a connection would put confidence-guided filtering on firmer footing and clarify which segments of a flow contribute most to epistemic uncertainty.

Our use of confidence is also primarily diagnostic. A natural next step is to study \emph{how} to act on uncertainty at inference time, through adaptive ODE solvers that refine integration where posterior variance is high, guidance schemes steering trajectories away from low-confidence regions, or selective re-sampling that allocates compute to the hardest trajectories.

%% file: assets/tables/table_a1_singlevsmulti.tex
\begin{tabular}{lrr}
\toprule
Method & Mispl.\% $\downarrow$ & KL $\downarrow$ \\
\midrule
Output-FMwC MAP & 4.3\% & 0.088 \\
FMwC MAP & 4.6\% & 0.084 \\
\bottomrule
\end{tabular}

%% file: assets/tables/table_a2_sampling_strategy.tex
\begin{tabular}{lrrr}
\toprule
Method & Stochastic $\downarrow$ & Mean vel ($k$=5) $\downarrow$ & MAP $\downarrow$ \\
\midrule
Output-FMwC & \textbf{4.3\%} & \textbf{4.3\%} & \textbf{4.3\%} \\
FMwC & 4.6\% & 4.6\% & 4.6\% \\
\bottomrule
\end{tabular}

%% file: assets/tables/table_a3_singlevsmulti_auprc.tex
\begin{tabular}{lrrr}
\toprule
Method & Scoring & AUPRC $\uparrow$ & Traj. \\
\midrule
FM & random & 0.06 & 1 \\
FMwC ($k$=5) & MC endpoint spread & 0.58 & 5 \\
FMwC & temporal confidence & 0.55 & 1 \\
FMwC & endpoint var & 0.44 & 1 \\
FMwC & traj-integrated & 0.33 & 1 \\
FMwC & analytical endpoint & 0.30 & 1 \\
\bottomrule
\end{tabular}

%% file: assets/tables/table_a4_propagation.tex
\begin{tabular}{lr}
\toprule
Propagation & Endpoint AUPRC $\uparrow$ \\
\midrule
1st-order & 0.30  \\
2nd-order & 0.30 \\
Gauss-Hermite & 0.30  \\
\bottomrule
\end{tabular}

%% file: assets/tables/table_a5_endpoint_vs_traj.tex
\begin{tabular}{lrr}
\toprule
Method & Endpoint AUPRC $\uparrow$ & Traj-int AUPRC $\uparrow$ \\
\midrule
FMwC & 0.30 & 0.33 \\
\bottomrule
\end{tabular}

%% file: assets/tables/table_a8_hutchinson_full.tex
\begin{tabular}{lrrrr}
\toprule
K & Per-step $\rho$ & Per-step $r$ & Integrated $\rho$ & Integrated $r$ \\
\midrule
1 & 0.500 & 0.410 & 0.947 & 0.919 \\
5 & 0.682 & 0.610 & 0.955 & 0.933 \\
10 & 0.745 & 0.674 & 0.955 & 0.934 \\
25 & 0.807 & 0.733 & 0.955 & 0.934 \\
50 & 0.835 & 0.762 & 0.955 & 0.934 \\
100 & 0.853 & 0.779 & 0.955 & 0.934 \\
250 & 0.864 & 0.790 & 0.955 & 0.934 \\
500 & 0.868 & 0.794 & 0.955 & 0.934 \\
\bottomrule
\end{tabular}

%% file: references.bib
@article{berry2023efficient,
  title={Efficient epistemic uncertainty estimation in regression ensemble models using pairwise-distance estimators},
  author={Berry, Lucas and Meger, David},
  journal={arXiv preprint arXiv:2308.13498},
  year={2023}
}

@article{gat2024discrete,
  title={Discrete flow matching},
  author={Gat, Itai and Remez, Tal and Shaul, Neta and Kreuk, Felix and Chen, Ricky TQ and Synnaeve, Gabriel and Adi, Yossi and Lipman, Yaron},
  journal={Advances in Neural Information Processing Systems},
  volume={37},
  pages={133345--133385},
  year={2024}
}

@article{miller2024flowmm,
  title={Flowmm: Generating materials with riemannian flow matching},
  author={Miller, Benjamin Kurt and Chen, Ricky TQ and Sriram, Anuroop and Wood, Brandon M},
  journal={arXiv preprint arXiv:2406.04713},
  year={2024}
}

@article{chen2024flow,
  title={Flow matching on general geometries},
  author={Chen, Ricky TQ and Lipman, Yaron},
  journal={arXiv preprint arXiv:2302.03660},
  year={2024}
}

@article{nalisnick2019deep,
  title={Do deep generative models know what they don't know?},
  author={Nalisnick, Eric and Matsukawa, Akihiro and Teh, Yee Whye and Gorur, Dilan and Lakshminarayanan, Balaji},
  journal={arXiv preprint arXiv:1810.09136},
  year={2019}
}

@article{lakshminarayanan2017simple,
  title={Simple and scalable predictive uncertainty estimation using deep ensembles},
  author={Lakshminarayanan, Balaji and Pritzel, Alexander and Blundell, Charles},
  journal={Advances in neural information processing systems},
  volume={30},
  year={2017}
}

@inproceedings{gal2016dropout,
  title={Dropout as a {B}ayesian approximation: Representing model uncertainty in deep learning},
  author={Gal, Yarin and Ghahramani, Zoubin},
  booktitle={International conference on machine learning},
  pages={1050--1059},
  year={2016}
}

@article{kingma2015variational,
  title={Variational dropout and the local reparameterization trick},
  author={Kingma, Diederik P and Salimans, Tim and Welling, Max},
  journal={Advances in neural information processing systems},
  volume={28},
  year={2015}
}

@article{molchanov2017variational,
  title={Variational dropout sparsifies deep neural networks},
  author={Molchanov, Dmitry and Ashukha, Arsenii and Vetrov, Dmitry},
  journal={International Conference on Machine Learning},
  year={2017}
}

@article{wang2013fast,
  title={Fast dropout training},
  author={Wang, Sida and Manning, Christopher},
  journal={International conference on machine learning},
  pages={118--126},
  year={2013}
}

@article{postels2019sampling,
  title={Sampling-free epistemic uncertainty estimation using approximated variance propagation},
  author={Postels, Janis and Ferroni, Francesco and Coskun, Huseyin and Navab, Nassir and Tombari, Federico},
  journal={International Conference on Computer Vision},
  year={2019}
}

@article{betala2025lemat,
  title={LeMat-GenBench: A Unified Evaluation Framework for Crystal Generative Models},
  author={Betala, Siddharth and Gleason, Samuel P and Ramlaoui, Ali and Xu, Andy and Channing, Georgia and Levy, Daniel and Fourrier, Cl{\'e}mentine and Kazeev, Nikita and Joshi, Chaitanya K and Kaba, S{\'e}kou-Oumar and others},
  journal={arXiv preprint arXiv:2512.04562},
  year={2025}
}

@article{coscia2025barnn,
  title={BARNN: A Bayesian Autoregressive and Recurrent Neural Network},
  author={Coscia, Dario and Welling, Max and Demo, Nicola and Rozza, Gianluigi},
  journal={arXiv preprint arXiv:2501.18665},
  year={2025}
}

@article{coscia2025blips,
  title={{BLIPs: Bayesian Learned Interatomic Potentials}},
  author={Coscia, Dario and de Haan, Pim and Welling, Max},
  journal={arXiv preprint arXiv:2508.14022},
  year={2025}
}

@incollection{cover1991network,
     title={Network Information Theory},
     author={Cover, Thomas M and Thomas, Joy A},
     booktitle={Elements of Information Theory},
     edition={1st},
     publisher={Wiley},
     pages={374--458},
     year={1991}
}

@article{jazbec2025generative,
  title={Generative Uncertainty in Diffusion Models},
  author={Jazbec, Metod and Wong-Toi, Eliot and Xia, Guoxuan and Zhang, Dan and Nalisnick, Eric and Mandt, Stephan},
  journal={arXiv preprint arXiv:2502.20946},
  year={2025}
}

@article{fort2019deep,
  title={{Deep ensembles: A Loss Landscape Perspective}},
  author={Fort, Stanislav and Hu, Huiyi and Lakshminarayanan, Balaji},
  journal={arXiv preprint arXiv:1912.02757},
  year={2019}
}

@article{tan2023single,
  title={{Single-model uncertainty quantification in neural network potentials does not consistently outperform model ensembles}},
  author={Tan, Aik Rui and Urata, Shingo and Goldman, Samuel and Dietschreit, Johannes CB and G{\'o}mez-Bombarelli, Rafael},
  journal={npj Computational Materials},
  volume={9},
  number={1},
  pages={225},
  year={2023},
  publisher={Nature Publishing Group UK London}
}

@inproceedings{blundell2015weight,
  title={{Weight Uncertainty in Neural Networks}},
  author={Blundell, Charles and Cornebise, Julien and Kavukcuoglu, Koray and Wierstra, Daan},
  booktitle={International conference on machine learning},
  pages={1613--1622},
  year={2015},
  organization={PMLR}
}

@Article{graves2011practical,
  title={{Practical Variational Inference for Neural Networks}},
  author={Graves, Alex},
  journal={Advances in Neural Information Processing Systems},
  volume={24},
  year={2011}
}

@article{ambrogioni2024thermodynamics,
  title={The statistical thermodynamics of generative diffusion models: Phase transitions, symmetry breaking and critical instability},
  author={Ambrogioni, Luca},
  journal={arXiv preprint arXiv:2310.17467},
  year={2024}
}

@article{stancevic2026information,
  title={The Information Dynamics of Generative Diffusion},
  author={Stan{\v{c}}evi{\'c}, Dejan and Ambrogioni, Luca},
  journal={Entropy},
  volume={28},
  number={2},
  pages={195},
  year={2026},
  publisher={MDPI}
}

@inproceedings{raya2023spontaneous,
  title={Spontaneous symmetry breaking in generative diffusion models},
  author={Raya, Gabriel and Ambrogioni, Luca},
  booktitle={Advances in Neural Information Processing Systems},
  volume={36},
  year={2023}
}

@inproceedings{pooladian2023multisample,
  title={Multisample Flow Matching: Straightening Flows with Minibatch Couplings},
  author={Pooladian, Aram-Alexandre and Ben-Hamu, Heli and Domingo-Enrich, Carles and Amos, Brandon and Lipman, Yaron and Chen, Ricky T. Q.},
  booktitle={International Conference on Machine Learning},
  volume={202},
  pages={28100--28127},
  year={2023}
}

@article{wood2025family,
  title={{UMA: A Family of Universal Models for Atoms}},
  author={Wood, Brandon M and Dzamba, Misko and Fu, Xiang and Gao, Meng and Shuaibi, Muhammed and Barroso-Luque, Luis and Abdelmaqsoud, Kareem and Gharakhanyan, Vahe and Kitchin, John R and Levine, Daniel S and others},
  journal={arXiv preprint arXiv:2506.23971},
  year={2025}
}

@article{bayes1763lii,
  title={{An Essay towards solving a Problem in the Doctrine of Chances. By the late Rev. Mr. Bayes, FRS communicated by Mr. Price, in a letter to John Canton, A.M.F.R.S.}},
  author={Bayes, Thomas},
  journal={Philosophical transactions of the Royal Society of London},
  year={1763},
  publisher={The Royal Society London}
}

@inproceedings{hinton1993keeping,
  title={{Keeping the Neural Networks Simple by Minimizing the Description Length of the Weights}},
  author={Hinton, Geoffrey E and Van Camp, Drew},
  booktitle={Conference on Computational Learning Theory},
  year={1993}
}

@inproceedings{kingma2014auto,
  title={{Auto-Encoding Variational Bayes}},
  author={Kingma, Diederik P and Welling, Max},
  booktitle={International Conference on Learning Representations},
  year={2014}
}

@inproceedings{papamarkou2024position,
  title={{Position: Bayesian Deep Learning is Needed in the Age of Large-Scale AI}},
  author={Papamarkou, Theodore and Skoularidou, Maria and Palla, Konstantina and Aitchison, Laurence and Arbel, Julyan and Dunson, David and Filippone, Maurizio and Fortuin, Vincent and Hennig, Philipp and Hern{\'a}ndez-Lobato, Jos{\'e} Miguel and others},
  booktitle={International Conference on Machine Learning},
  year={2024}
}

@article{wilson2020bayesian,
  title={{Bayesian Deep Learning and a Probabilistic Perspective of Generalization}},
  author={Wilson, Andrew G and Izmailov, Pavel},
  journal={Advances in neural information processing systems},
  volume={33},
  pages={4697--4708},
  year={2020}
}

@inproceedings{izmailov2021bayesian,
  title={{What are Bayesian Neural Network Posteriors Really Like?}},
  author={Izmailov, Pavel and Vikram, Sharad and Hoffman, Matthew D and Wilson, Andrew Gordon Gordon},
  booktitle={International conference on machine learning},
  pages={4629--4640},
  year={2021},
  organization={PMLR}
}

@article{eijkelboom2024variational,
  title={Variational flow matching for graph generation},
  author={Eijkelboom, Floor and Bartosh, Grigory and Andersson Naesseth, Christian and Welling, Max and van de Meent, Jan-Willem},
  journal={Advances in Neural Information Processing Systems},
  volume={37},
  pages={11735--11764},
  year={2024}
}

@article{jain2013materials,
     title={The Materials Project: A materials genome approach to accelerating materials innovation},
     author={Jain, Anubhav and Ong, Shyue Ping and Hautier, Geoffroy and Chen, Wei and Richards, William Davidson and Dacek, Stephen and Cholia, Shreyas and Gunter, Dan and Skinner, David and Ceder, Gerbrand and Persson, Kristin A},
     journal={APL Materials},
     volume={1},
     number={1},
     pages={011002},
     year={2013},
     publisher={AIP Publishing}
   }

@article{qin2024defog,
  title={Defog: Discrete flow matching for graph generation},
  author={Qin, Yiming and Madeira, Manuel and Thanou, Dorina and Frossard, Pascal},
  journal={arXiv preprint arXiv:2410.04263},
  year={2024}
}

@article{lipman2022flow,
  title={Flow matching for generative modeling},
  author={Lipman, Yaron and Chen, Ricky TQ and Ben-Hamu, Heli and Nickel, Maximilian and Le, Matt},
  journal={arXiv preprint arXiv:2210.02747},
  year={2022}
}

@article{kamkari2024geometric,
  title={A geometric explanation of the likelihood OOD detection paradox},
  author={Kamkari, Hamidreza and Ross, Brendan Leigh and Cresswell, Jesse C and Caterini, Anthony L and Krishnan, Rahul G and Loaiza-Ganem, Gabriel},
  journal={arXiv preprint arXiv:2403.18910},
  year={2024}
}
